%% file: arxiv.tex
\documentclass{article}

\usepackage[total={6.5in,8.75in},
top=1.2in, left=0.9in, includefoot]{geometry}
\usepackage{microtype}
\usepackage{graphicx}
\usepackage{subfigure}
\usepackage[makeroom]{cancel}
\usepackage{booktabs} 
\usepackage{makecell}
\usepackage[colorlinks=true,citecolor=blue]{hyperref}
\usepackage{multirow}
\usepackage[table, svgnames, dvipsnames]{xcolor}
\usepackage{makecell, cellspace, caption}
\usepackage{booktabs,caption}
\usepackage[flushleft]{threeparttable}
\usepackage{amsmath, nccmath}
\usepackage{amssymb}
\usepackage{mathtools}
\usepackage{enumitem}
\usepackage{amsthm}
\usepackage{lipsum} 
\usepackage[capitalize,noabbrev]{cleveref}
\usepackage{cite}

\newtheorem{theorem}{Theorem}

\newtheorem{lemma}{Lemma}
\newtheorem{assumption}{Assumption}

\theoremstyle{remark} 
\newtheorem{remark}{Remark}

\usepackage[round]{natbib}

\usepackage{algorithm}
\usepackage{algorithmic}
\usepackage{microtype}
\usepackage{graphicx}
\usepackage{subfigure}
\usepackage[makeroom]{cancel}
\usepackage{booktabs} 
\usepackage{makecell}
\usepackage{float}
\usepackage{caption} 
\usepackage[table, svgnames, dvipsnames]{xcolor}
\definecolor{myred}{RGB}{192,0,0}
\definecolor{myblue}{RGB}{0,0,128}
\usepackage[flushleft]{threeparttable}
\usepackage{amsmath, nccmath}
\usepackage{amssymb}
\usepackage{mathtools}
\usepackage{enumitem}
\usepackage{amsthm}
\usepackage{lipsum} 
\usepackage[capitalize,noabbrev]{cleveref}

\input{macros}
\input{math_commands}

\title{Distributed Sign Momentum with Local Steps for Training Transformers}

\author{Shuhua Yu\thanks{Carnegie Mellon University. Emails: \texttt{\{shuhuay,soummyak\}@andrew.cmu.edu}. SY's work was partially done during an internship at ByteDance Inc.} 
\and Ding Zhou\thanks{ByteDance Inc. Emails: \texttt{\{ding.zhou,cong.
xie,an.xu,zhangzhi.joshua,liuxin.ai\}@bytedance.com}} 
\and Cong Xie\footnotemark[2]
\and An Xu\footnotemark[2]
\and Zhi Zhang\footnotemark[2]
\and Xin Liu\footnotemark[2]
\and Soummya Kar\footnotemark[1]
}

\date{    \footnotemark[1]~Carnegie Mellon University\\[0.2ex]%
    \footnotemark[2]~ByteDance Inc. \\[2ex]%
    November 2024; revised March 2025} 

\begin{document}

\maketitle  

\begin{abstract}
Pre-training Transformer models is resource-intensive, and recent studies have shown that sign momentum is an efficient technique for training large-scale deep learning models, particularly Transformers.  However, its application in distributed training remains underexplored. This paper investigates a novel communication-efficient distributed sign momentum method with multiple local steps, to cope with the scenarios where communicating at every step is prohibitive. Our proposed method allows for a broad class of base optimizers for local steps, and uses sign momentum in the global step, where momentum is generated from differences accumulated during local steps. For generic base optimizers, by approximating the sign operator with a randomized version that acts as a continuous analog in expectation, we present a general convergence analysis, which specializes to an $O(1/\sqrt{T})$ rate for a particular instance. When local step is stochastic gradient descent, we show an optimal $O(1/T^{1/4})$ rate in terms of $\ell_1$ gradient norm for nonconvex smooth cost functions. We extensively evaluate our method on the pre-training of various sized GPT-2 models from scratch, and the empirical results show significant improvement compared to other distributed methods with multiple local steps. 
\end{abstract}

\section{Introduction}
In this paper, we tackle the distributed optimization problem across $n$ workers:
\begin{align}
\label{eq:dist-opt}
\minimize_{\x \in \R^d} f(\x) := \frac{1}{n}\sum_{i = 1}^n \E_{\bxi^{(i)} \sim \cD_i} f_i (\x, \bxi^{(i)}),
\end{align}
where $f_i$ represents the cost function of a supervised learning model, \(\x\) denotes the model parameters, and \( \bxi^{(i)} \) refers to the data samples drawn from data distribution $\cD_i$ on worker $i \in [n]$. In distributed training tasks, $\cD_i$ can be a finite collection of samples allocated on worker $i$.

The problem formulation \eqref{eq:dist-opt} serves as a natural abstraction for distributed training \citep{lian2017can, mcmahan2017communication}. A common approach to solving \eqref{eq:dist-opt} involves workers first computing mini-batch gradients of $f_i$, followed by an all-reduce step to aggregate the local mini-batch gradients and obtain the global average. A global optimization step, such as Stochastic Gradient Descent (SGD) \citep{robbins1951stochastic}, Momentum \citep{polyak1964some}, Nesterov’s Accelerated Gradient (NAG) \citep{nesterov1983method}, Adam \citep{kingma2014adam} and its weight decay variant AdamW \citep{loshchilov2017decoupled}, or more recently Lion \citep{chen2024symbolic}, is then applied. This method is well suited for distributed learning scenarios where communication is inexpensive. 

However, in large-scale distributed training across GPU clusters, synchronized communication during the all-reduce step can become a bottleneck due to stragglers \citep{dean2012large} or slow inter-node and inter-cluster communication, making per-step communication impractical. Several approaches have been proposed to mitigate this communication overhead. Decentralized methods, such as gossiping SGD \citep{jin2016scale} and consensus-based distributed SGD \citep{jiang2017collaborative}, use approximate mini-batch gradient averaging instead of exact averaging to reduce idle time. Local step methods, like local SGD \citep{lindon, gorbunov2021local}, local momentum (PR-SGD-Momentum) \citep{yu2019linear} and local Adam \citep{lumaximizing}, reduce communication frequency to save on bandwidth. Additionally, some local step methods further optimize performance by overlapping communication and computation \citep{wang2019slowmo, sun2024co2}, utilizing stale synchronized model parameters for local steps to reduce the waiting time for synchronizing the latest local models. All these approaches, although reduce communication and run faster, add perturbations to the optimization process. These perturbations can lead to decreased performance in terms of training and validation losses under the same amount of computation \citep{wang2020tackling}. In this work, we focus on designing local step methods with excellent communication-performance tradeoffs. 


Several methods have been proposed to enhance the performance of local steps techniques. SCAFFOLD \citep{karimireddy2020scaffold} introduces a control variate that corrects the local SGD direction using global information, eliminating the dependency on local gradient dissimilarity in the convergence guarantees. One of the most relevant methods, BMUF \citep{chen2016scalable}, combines local steps with a global momentum step to accelerate distributed training, without the need for additional communication (which can also be a bottleneck in high-dimensional problem), such as transmitting the control variate in SCAFFOLD or the local momentum in PR-SGD-Momentum \citep{yu2019linear}. SlowMo \citep{wang2019slowmo} extends BMUF by providing a general framework that operates on local base optimizers, such as local SGD and SGP \citep{assran2019stochastic}, and periodically applies a global momentum step after running the local optimizer for several iterations. 


Beyond local step methods, momentum has long been recognized as a key technique for improving deep learning training, enhancing both optimization and generalization \citep{sutskever2013importance}. More recently, \citet{kunstner2023noise} suggest that sign momentum may be the key factor behind Adam’s significant advantage over SGD in training Transformers \citep{vaswani2017attention}.  \citet{chen2024symbolic} introduced the Lion (Evolved Sign Momentum) optimization method, discovered through program search, which outperforms the widely used Adam \citep{kingma2014adam} in various vision and autoregressive language modeling tasks. Furthermore, \citet{kunstner2023noise}, through experiments with different batch sizes, suggest that Adam’s significant performance advantage over SGD likely stems from its similarity to sign descent with momentum, particularly in Transformers used for language models. 

Inspired by the effectiveness of the sign momentum method in training large neural networks, particularly Transformers, and the success of global momentum steps in local step methods, a natural question is:
\begin{center}
\textit{Can we extend the sign momentum method to distributed training of Transformers with local steps?}
\end{center}
\subsection{Contributions} 
In this work, we answer this question affirmatively by proposing a framework for distributed sign momentum with local steps, supported by both empirical evaluations and theoretical justifications. Our contributions are as follows: 
\begin{enumerate}[label=(\alph*)]
    \item We introduce a new framework for distributed sign momentum with local steps. This framework is versatile, allowing any proper base optimizers for local steps, while using differences accumulated from these updates to adjust the momentum in the global sign momentum step. Details are provided in Algorithm \ref{alg:dist-sign-mom}.
    \item For generic base optimizer, we provide a general theoretical analysis of the proposed framework by approximating the sign operator with a randomized continuous analog in expectation. When specializing the base optimizer as SGD, we obtain an in-expectation convergence rate of $O\big(1/\sqrt{T}\big)$ in terms of squared Euclidean norms, where $T$ is the total global steps.
    \item When the base optimizer is SGD and the actual sign operator is used, we show an optimal in-expectation rate of $O(1/T^{1/4})$ measured in $\ell_1$ gradient norms, matching the best-known rate of centralized sign momentum method. 
    \item We extensively evaluate the performance of Algorithm \ref{alg:dist-sign-mom} in pre-training GPT-2 models of various sizes from scratch, demonstrating consistent improvements over other distributed methods with multiple local steps and highlighting its strength in large-scale scenarios where communicating at every step is prohibitively costly. 
\end{enumerate} 

\subsection{Related work} 
\textit{Communication compression}. In distributed optimization, there has been extensive research on compressing communication between workers to reduce overhead. This includes methods such as 1-bit sign SGD \citep{seide20141, bernstein2018signsgd}, gradient difference compression \citep{gorbunov2021marina}, compression with error-feedback \citep{karimireddy2019error}, and compression with error reset \citep{xie2020cser}, among others. Although this work focuses on local step methods without communication compression, several 1-bit compression (sign-based) methods, such as SIGNUM \citep{bernsteinsignsgd} and Major Vote signSGD \citep{sun2023momentum}, apply (random) sign operations to local momentum and use majority voting for local signs. In contrast, Algorithm \ref{alg:dist-sign-mom} employs full-precision communication during the global synchronization step, with the sign taken after aggregation. 

\textit{signSGD and its momentum variants}. The convergence of signSGD and its momentum variants has been studied for nonconvex and smooth functions. For signSGD, defined as:
\begin{align}
\label{eq:signsgd}
    \x_{t+1} = \x_t - \eta_t \text{sign}(\nabla f(\x_t, \bxi_t)), 
\end{align}
an $O(T^{-1/4})$ convergence rate  has been established under an increasing batch size \citep{bernstein2018signsgd}, or under unimodal and symmetric stochastic gradients \citep{bernsteinsignsgd}. For signSGD with a momentum buffer \(\m_t\), given by: 
\begin{align}
\begin{split}
\label{eq:signmom}
\m_{t +1} = \beta \m_t + (1 - \beta)\nabla f(\x_t, \bxi_t),  \\
\x_{t +1} = \x_t - \eta_t \text{sign}( \m_{t +1}), 
\end{split}
\end{align}
its $O(1/T^{-1/4})$ convergence rate is guaranteed under a decreasing step size and increasing batch size, as shown by \citet{bernstein2018signsgd}. \citet{sun2023momentum} improve the previous results by showing the same convergence rate under a \textit{constant} batch size independent of $T$. Additionally, \citet{karimireddy2019error} shows that signSGD with error-feedback, another momentum variant, also achieves an $O(T^{-1/4})$ convergence rate, albeit under a more stringent bounded stochastic gradient condition. 

\textit{Distributed sign momentum.} The idea of distributed sign momentum is also studied in works including \citet{liu2024communication} and \citet{tang2024fedlion}. \citet{tang2024fedlion} uses Lion for multiple local steps and average model and momentum periodically, and \citet{liu2024communication} computes the momentum term in Lion locally for once and uses global aggregation for model updates. Conversely, our proposed framework supports multiple local steps using any suitable base optimizer to generate local differences, which are then leveraged for sign momentum in the global step.

\section{Distributed sign momentum} 
We introduce a general framework, outlined in Algorithm \ref{alg:dist-sign-mom}, that can incorporate any off-the-shelf base optimizer within the inner loop. Each worker $i \in [n]$ maintains a local model copy $\x_{t, k}^{(i)}$ for the outer iteration $t$ and inner iteration $k$, denoted as $(t, k)$. Specifically, at $t$-th outer iteration, each worker $i \in [n]$ performs $\tau$ local model updates: 
\begin{align} 
\label{eq:local-steps}
    \x_{t, k+ 1}^{(i)} = \x_{t, k}^{(i)} - \gamma_t \vd^{(i)}_{t, k}, k = 0, \ldots, \tau -1,
\end{align}
where $\gamma_t$ is the local learning rate (LR) and $\vd^{(i)}_{t, k}$ is the update direction of local base optimizer. We allow for any base optimizers, such as SGD (with or without momentum), AdamW, and others. For example, if the base optimizer is mini-batch SGD, then
\begin{align}
\label{eq:base-sgd}
\vd_{t, k}^{(i)} = \nabla f_i(\x_{t, k}^{(i)}, \bxi_{t, k}^{(i)}), 
\end{align}
where \(\bxi_{t, k}^{(i)}\) is a sample drawn from \(\cD_i\). As in \citet{wang2019slowmo}, base optimizers can also involve communication steps, such as SGP \citep{assran2019stochastic}, and momentum steps. We denote the final update direction applied to worker $i$  at iteration  $(t, k)$  as \( \vd_{t,k}^{(i)} \). 

After $\tau$ local steps using the base optimizer, the workers communicate to perform a global sign momentum step. Two global buffers are maintained: the model buffer \(\x_{t, 0}\) and the momentum buffer \(\m_{0}\). All workers are assumed to be initialized with \(\x_{0, 0}\) and \(\m_{0} = \zero\). First, an all-reduce operation is performed to compute the average of the local models, i.e., \(\x_{t, \tau} = (1/n)\sum_{i=1}^n \x_{t, \tau}^{(i)}\). The momentum buffer is then used to execute the global sign momentum step (line 9), for some $\beta_1 \in [0, 1]$, 
\begin{align}
    \vu_{t + 1} & = \beta_1 \m_t + \frac{1- \beta_1}{\gamma_t} (\x_{t, 0} - \x_{t, \tau}), \label{eq:ut-update} \\
    \x_{t + 1, 0} & = \x_{t, 0} - \eta \gamma_t \big( \text{sign}(\vu_{t +1}) + \lambda \x_{t, 0} \big), \label{eq:model-update}
\end{align}
where $\eta$ is the global learning rate and $\lambda$ is the decoupled weight decay. Then, we perform the following global momentum update (line 8), for some $\beta_2 \in [0, 1]$, 
\begin{align}
\m_{t + 1} = \beta_2 \m_{t} + \frac{1 - \beta_2}{\gamma_t}(\x_{t, 0} - \x_{t, \tau}). \label{eq:mom-update}
\end{align}  
Note that \eqref{eq:ut-update}-\eqref{eq:mom-update} mimics the update rule of Lion \citep{chenlion} (also described in Section \ref{sec:optimizers} in Appendix), by treating the differences accumulated from local steps as pseudo-stochastic gradients. In both \eqref{eq:ut-update} and \eqref{eq:mom-update}, we scale \((\x_{t, 0} - \x_{t, \tau})\) by $\frac{1}{\gamma_t}$ to make the momentum buffer independent of the learning rate $\gamma_t$, which can vary over time when using a learning rate schedule.

\textit{Algorithm instances}. Instances of Algorithm \ref{alg:dist-sign-mom} can be obtained by specifying the parameters $\eta, \gamma_t, \beta_1, \beta_2, \lambda$. For example, setting $\beta_1 = \beta_2 = \beta, \lambda = 0, \tau = 1$ recovers signSGD with momentum as described in \eqref{eq:signmom}. In particular, when $n = 1$, Algorithm \ref{alg:dist-sign-mom} simplifies to a signed Lookahead optimizer \citep{zhang2019Lookahead} with decoupled weight decay. 


\begin{algorithm}[ht]
\caption{Distributed Sign Momentum}
\label{alg:dist-sign-mom}
\begin{algorithmic}[1]
    \REQUIRE Local update directions $\vd_{t, k}^{(i)}$, local learning rate $\gamma$, global learning rate $\eta$, momentum coefficients $\beta_1, \beta_2$. 
    \STATE Initialize \( \x_{0, 0} \), let \( \m_0 = 0. \)
    \FOR{\( t = 0 \) to \( T - 1 \)}
        \FOR{\(i = 1 \) to \( n \) in parallel}
            \FOR{ \( k = 0 \) to \( \tau - 1\)}
            \STATE \( \x^{(i)}_{t, k+1} \gets \x^{(i)}_{t, k} - \gamma_t \vd_{t, k}^{(i)} \)
        \ENDFOR
        \ENDFOR
        \STATE All-reduce step: $\x_{t, \tau} = \frac{1}{n} \sum_{i=1}^n \x^{(i)}_{t, \tau}$.   
        \STATE Update global parameter:  \\ 
        \( \qquad \bu_{t + 1} \gets  \beta_1 \m_{t} + \frac{1 - \beta_1}{\gamma_t }(\x_{t, 0} -\x_{t, \tau}) \) \\ 
        \( \qquad \x_{t + 1, 0} \gets \x_{t, 0} - \eta \gamma_t \big( \text{sign}(\bu_{t+1}) + \lambda \x_{t, 0} \big) \) \\
        \STATE Update momentum: \\
        \( \qquad \m_{t + 1} = \beta_2 \m_t + \frac{1- \beta_2}{\gamma_t}(\x_{t, 0} - \x_{t, \tau})  \) \\
        \STATE Synchronize over all workers: \\
        \( \qquad \forall i \in [n], \ \x^{(i)}_{t + 1, 0} \gets \x_{t+ 1, 0} \)
    \ENDFOR
  \STATE \textbf{return} \( \x_{T, 0} \)
\end{algorithmic}
\end{algorithm} 

\textit{Global sign momentum step}. The updates \eqref{eq:ut-update}-\eqref{eq:mom-update} offer a flexible configuration for the global sign momentum step. Specifically, when $\beta_2 > \beta_1$, the current \((\x_{t, 0} - \x_{t, \tau})\) takes on a larger weight in the exponential moving average, similar to classical momentum as in SlowMo. This adjustment in momentum  can lead to further acceleration, as studied by \citet{shi2022understanding}. Additionally, the decoupled weight decay in \eqref{eq:model-update} is similar to the weight decay used in AdamW, making it more advantageous than Adam.


\textit{Distributed optimizer}. We note that Algorithm \ref{alg:dist-sign-mom} can be seamlessly integrated with the ZERO-series distributed optimizer \citep{rajbhandari2020zero} to handle large models. Consider a layered distributed optimization strategy: GPUs on a single node can be treated as a joint worker, utilizing a distributed optimizer like ZERO2 for local steps of the base optimizer, while the global buffers \(\x_{t, 0}\) and \(\m_{t, 0}\) are distributed across nodes. This approach allows more frequent communication for local steps to leverage faster intra-node communication, such as NVLinks.

\section{Convergence analysis} 
We provide some theoretical justifications on the convergence of Algorithm \ref{alg:dist-sign-mom} under some regular conditions. Let $f_i(\x) := \E_{\bxi^{(i)} \sim \cD_i} f_i (\x, \bxi^{(i)})$. We make the following assumptions on local functions.

\begin{assumption}
\label{as:smooth}
Each local function $f_i$ is differentiable and $\forall \x, \y \in \R^d, \| \nabla f_i(\x) - \nabla f_i(\y)\| \le L \| \x - \y \|$ for some $L > 0$. 
\end{assumption}

Let $\E_{t, k}$ denote the conditional expectation given all the historical randomness up to iteration $(t, k)$. We define the virtual global averages: 
\begin{align*}
    \x_{t, k} = \frac{1}{n}\sum_{i=1}^n \x_{t, k}^{(i)}, \ \vd_{t, k} = \frac{1}{n}\sum_{i = 1}^n \vd_{t, k}^{(i)}. 
\end{align*} 
We assume that the average update direction $\vd_{t, k}$ has bounded variance. 
\begin{assumption}
\label{as:noise}
There exists some $\zeta > 0$ such that $\E [ \| \vd_{t, k}  - \E_{t, k}[\vd_{t, k}]\|^2 ] \le \zeta^2$. 
\end{assumption}

Similar to \cite{sun2023momentum}, we assume that update directions of base optimizers are uniformly bounded throughout the optimization process.
\begin{assumption}
\label{as:bdd-sg}
The local update direction of each local base optimizer is bounded at at all iteration $(t, k)$: $\forall (t, k), \forall i \in [n], \ \| \vd_{t, k}^{(i)} \|^2 \le R^2$ for some $R > 0$. 
\end{assumption} 

\subsection{Randomized sign operator}
For a generic base optimizer that generates update directions satisfying Assumption \ref{as:noise} and \ref{as:bdd-sg}, we present an general analysis that  approximates the sign operator with a randomized analog. For any vector \( \vv = [v_1, \ldots, v_d]^\top \in \R^d \) such that \( \| \vv \| \leq B \), we use the randomized sign operator \(\cS_r(\vv)\), where the $j$-th component is defined as:
\begin{align}
\label{eq:rand-sign-1}
    [\cS_r(\vv)]_j = \begin{cases}
        -\sign(v_j), & \text{with probability (w.p.)} \frac{1}{2} - \frac{|v_j|}{2B}, \\
        \sign(v_j), & \text{ w.p. } \frac{1}{2} + \frac{|v_j|}{2B}.
    \end{cases}
\end{align}  
Or alternatively, 
\begin{align}
\label{eq:rand-sign-2}
    [\cS_r(\vv)]_j = \begin{cases}
        0, & \text{ w.p. } 1 - \frac{|v_j|}{B}, \\
        \sign(v_j), & \text{ w.p. } \frac{|v_j|}{B}.
    \end{cases}
\end{align}
We use these randomized sign operators as a linear and continuous analog of the original sign operator in expectation as in \citet{safaryan2021stochastic}. Let \(\E_{\cS}\) denote the expectation taken over the randomness of the randomized sign operator, we have the following property. 
\begin{lemma}
\label{lm:signopr}
For a random vector $\bv \in \R^d$ that satisfies $\|\vv\| \le B$ almost surely, we have $\E_{\cS}[\vv] = \vv/B$, and $\E_{\cS}[\| \cS_r(\vv) - \vv/B \|^2] \le d$. 
\end{lemma}  
With the above linerization in expectation, we establish a general convergence for some specific instances of Algorithm \ref{alg:dist-sign-mom}. Proofs of all theorems  are deferred to the Appendix. 

\begin{theorem}
\label{thm:rands-main}
Let Assumptions \ref{as:smooth}, \ref{as:noise}, and \ref{as:bdd-sg} hold. Run Algorithm \ref{alg:dist-sign-mom} by approximating sign with $\cS_r$ in \eqref{eq:rand-sign-1} or \eqref{eq:rand-sign-2} with $B = \tau R$. Take local learning rate $\gamma = \frac{R}{\eta}\sqrt{ \frac{n\tau}{T}}$, weight decay $\lambda = 0$, and momentum coefficients $\beta_1 = \beta_2 = \beta \in [0, 1)$. When the outer iterations number $T \ge 4nL^2 \big[ 4 (\tau - 1)\big( \frac{\tau R}{\eta} - 1 \big)^2 + \frac{8\tau \beta^2}{(1 - \beta)^2} + 1 \big]$, the iterates $\{\x_{t, k}\}$ generated from Algorithm \ref{alg:dist-sign-mom} satisfy that 
\begin{align*}
& \frac{1}{\tau T} \sum_{ t= 0}^{T - 1}\sum_{k = 0}^{\tau - 1} \E \| \nabla f(\x_{t, k} )\|^2 \le \frac{2(f(\x_{0, 0})- f_*)}{ \sqrt{n  \tau T}} \\
& \quad + \underbrace{\frac{1}{\tau T} \sum_{t = 0}^{T - 1}\sum_{k = 0}^{\tau - 1} \E [ \| \nabla f(\x_{t, k}) - \E[\vd_{t, k}] \|^2 ]}_{\text{Effect of base optimizer}}   + \zeta^2 L \sqrt{ \frac{n}{ \tau T }} \\
& \quad +  \frac{4n L^2 \zeta^2}{T}\Big[ \big(\frac{\tau R}{\eta} - 1 \big)^2 + \frac{2 \beta^2}{1 - \beta^2} \Big] + d L R^2 \sqrt{ \frac{n \tau}{T}}. 
\end{align*}
\end{theorem}

Theorem \ref{thm:rands-main} provides a general convergence guarantee for the simple sign momentum case of Algorithm \ref{alg:dist-sign-mom}. The convergence rate depends on the \textit{effect of base optimizer}, which can measure both gradient heterogeneity across workers and stochastic gradient bias. We next specialize to the case where the base optimizer is SGD. 

\begin{table*}[t]
    \centering
    \caption{Model configurations and peak learning rates}
    \label{tbe:gpt2-configs}
    \begin{tabular}{|c|c|c|c|c|c|}
        \hline
         & Size & Embedding dimension & \# of heads & Depth & Peak LR \\ \hline 
        Small     & 125M   & 768   & 12  & 12 & 5e-4 \\ \hline
        Medium    & 355M   & 1024  & 16  & 24 & 2e-4 \\ \hline
        Large     & 770M   & 1280  & 20  & 36 & 2e-4 \\ \hline
    \end{tabular}
\end{table*} 

\begin{theorem}
\label{thm:base-sgd}
Let Assumptions \ref{as:smooth}, \ref{as:bdd-sg} hold, and in addition assume that: 
\begin{enumerate}[label=(\alph*)]
    \item the local update direction $\vd_{t, k}^{(i)} = \nabla f_i(\x_{t,k}^{(i)}, \bxi_{t,k}^{(i)})$, and $\forall \x \in \R^d, \E_{\bxi^{(i)} \sim \cD_i} [\| \nabla f_i(\x, \bxi^{(i)}) - \nabla f_i(\x) \|^2 ] \le \sigma^2$ for some $\sigma > 0$; 
    \item  $\forall \x \in \R^d, (1/n)\sum_{i=1}^n \| \nabla f(\x) - \nabla f_i(\x) \|^2 \le \delta^2$ for some $\delta > 0$.
\end{enumerate} 
Then, by approximating sign with randomized operator $\cS_r$ and using parameters $\gamma, \eta, \lambda, \beta_1, \beta_2$ and $T$ as in Theorem \ref{thm:rands-main}, we have $\frac{1}{\tau T} \sum_{t=0}^{T - 1} \sum_{k=0}^{\tau - 1} \E [ \| \nabla f(\x_{t, k}) \|^2 ] = O(1/\sqrt{ T})$, specifically, it 
\begin{align*}
    & \frac{1}{\tau T} \sum_{ t= 0}^{T - 1}\sum_{k = 0}^{\tau - 1} \E \| \nabla f(\x_{t, k} )\|^2 \le \frac{2(f(\x_{0, 0})- f_*)}{ \sqrt{n  \tau T}} \\
    & \quad + \frac{3 n \tau^2 L^2 R^2 (\sigma^2 + 3 \tau \delta^2)}{\eta^2 T}  + \sigma^2 L \sqrt{ \frac{n}{ \tau T }}  \\
    & \quad + \frac{4n L^2 \sigma^2}{T}\Big[ \big(\frac{\tau R}{\eta} - 1 \big)^2 + \frac{2 \beta^2}{1 - \beta^2} \Big] + dLR^2\sqrt{\frac{n \tau}{T}}. 
\end{align*}
\end{theorem}

\begin{remark}
We compare the instance in Theorem \ref{thm:base-sgd} with a related algorithm, Federated Majority Vote signSGD with Stochastic Simple Momentum (Federated MV-sto-signSGD-SIM), analyzed by \citet{sun2023momentum}. Federated MV-sto-signSGD-SIM incorporates similar algorithmic concepts but remains distinct from Algorithm \ref{alg:dist-sign-mom}. First, our algorithm uses aggregated local differences $(\x_{t, 0} - \x_{t, \tau})$ for a global  momentum update while MV-sto-signSGD-SIM uses local stochastic gradients to update local momentum. Second, \citet{sun2023momentum} uses randomized sign to compress worker to server communications, i.e., sign based majority vote,  while we use full-precision communication among workers.
\end{remark}

\subsection{Sign operator}
When the base optimizer is SGD, we also present its convergence analysis using actual sign for Algorithm \ref{alg:dist-sign-mom}.
\begin{theorem}
\label{thm:hard_sign}
Let Assumptions \ref{as:smooth}, \ref{as:bdd-sg} hold. In addition, the  local update on worker $i$ satisfies that $\vd_{t, k}^{(i)} = \nabla f_i(\x_{t,k}^{(i)}, \bxi_{t,k}^{(i)})$, and $\forall \x \in \R^d, \E_{\bxi^{(i)} \sim \cD_i} [\| \nabla f_i(\x, \bxi^{(i)}) - \nabla f_i(\x) \|^2 ] \le \sigma^2$ for some $\sigma > 0$. Taking $\eta = \frac{1}{L T^{3/4}}$, and $1 - \beta = \frac{1}{\sqrt{T}}$, it holds that
\begin{align*}
     & \frac{1}{T}\sum_{t = 0}^{T  - 1} \E \big[ \| \nabla f(\x_{t, 0}) \ \|_1 \big]  \le \frac{L(f(\x_{0, 0}) - f_*)}{\gamma  T^{1/4}} \\
     & \quad + \frac{2\sqrt{d}}{T^{1/2}} \|\nabla f(\x_{0, 0}) \| + \frac{2d \gamma}{T^{1/4}} \\
     & \quad + \frac{2\sigma}{T^{1/4}} \sqrt{\frac{d}{\tau n} } +  \frac{\sqrt{d} \tau R + \gamma d / 2}{T^{3/4}}.
\end{align*} 
\end{theorem}

\begin{remark}
The convergence rate above in terms of $\ell_1$ gradient norm matches the rate of centralized sign momentum method, namely, update in \eqref{eq:signmom}, as shown in \citet{sun2023momentum}.  Specifically, in one of the leading terms $\frac{2\sigma}{T^{1/4}}\sqrt{\frac{d}{\tau n}}$, the benefit of multiple local steps and linear speedup are demonstrated, and it suggests our algorithm is preferable in large noise regime. We note that \citet{sun2023momentum} also shows similar convergence rate of federated MV-sto-signSGD-SIM), but only to an $O(\frac{dR}{\sqrt{n}})$ neighborhood under the same Assumption \ref{as:bdd-sg}, due to its use of randomized signs for worker to server communication. 
\end{remark}

\section{Experiments}

We evaluate Algorithm \ref{alg:dist-sign-mom} on auto-regressive language modeling using GPT-2 \citep{radford2019language}, with model sizes ranging from 125M to 770M parameters, trained from scratch on the OpenWebText dataset \citep{gokaslan2019openwebtext}. Following standard practice, we set the context length of GPT-2 to 1024 tokens. We consider three versions of GPT-2: 125M (Small), 355M (Medium), and 770M (Large) parameters, with configurations as detailed in Table \ref{tbe:gpt2-configs}.

\begin{figure*}[ht]
    \centering
    \begin{minipage}{0.32\textwidth}
        \centering
        \includegraphics[width=\linewidth]{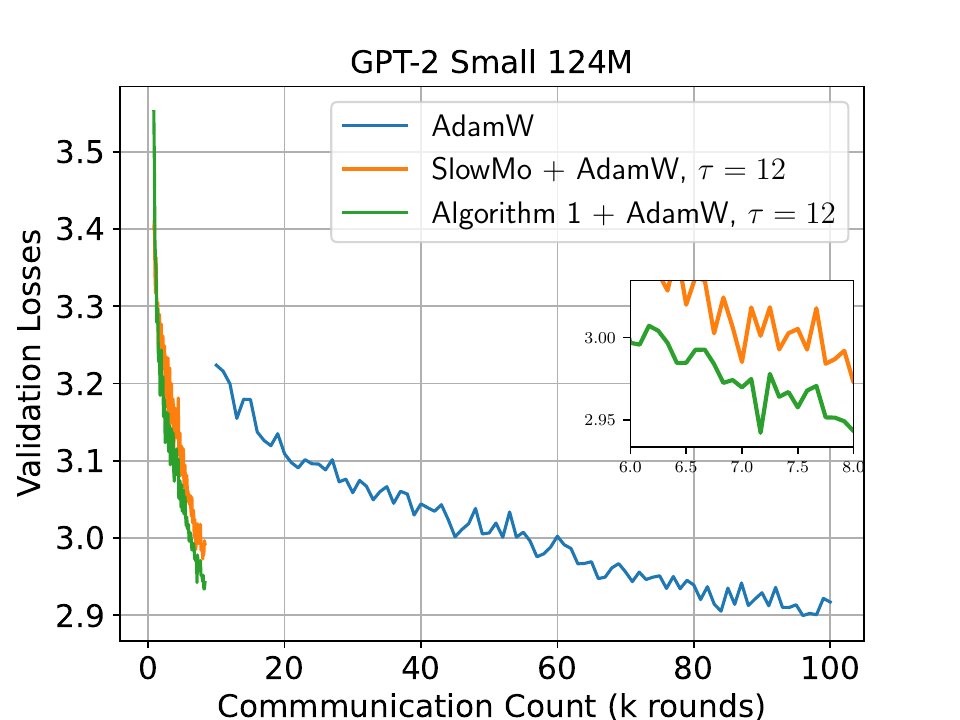}
    \end{minipage}\hfill
    \begin{minipage}{0.32\textwidth}
        \centering
        \includegraphics[width=\linewidth]{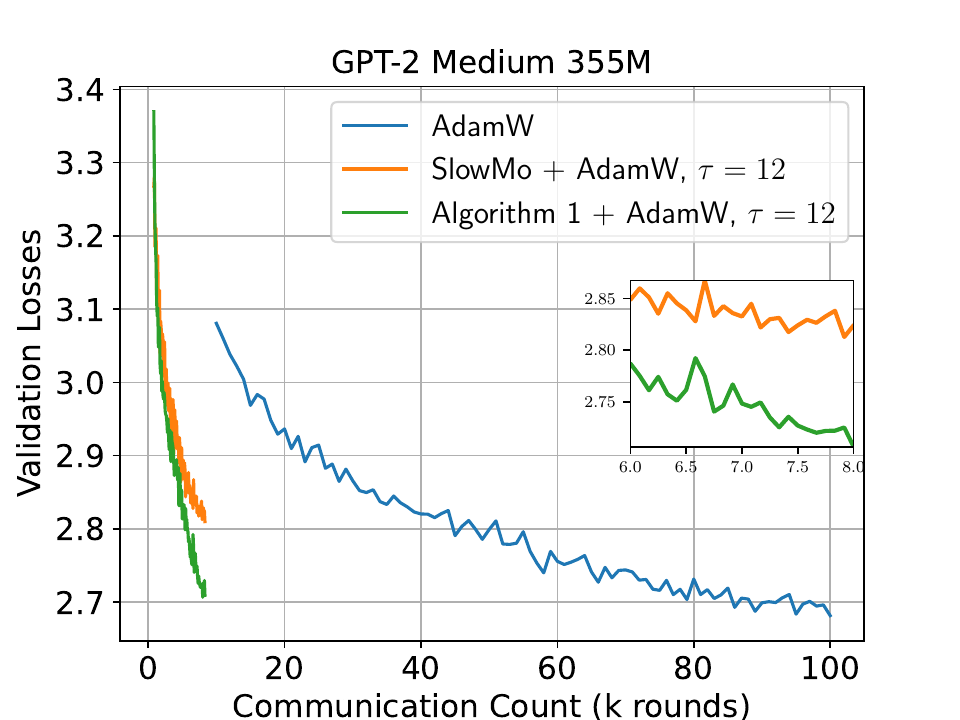}
    \end{minipage}\hfill
    \begin{minipage}{0.32\textwidth}
        \centering
        \includegraphics[width=\linewidth]{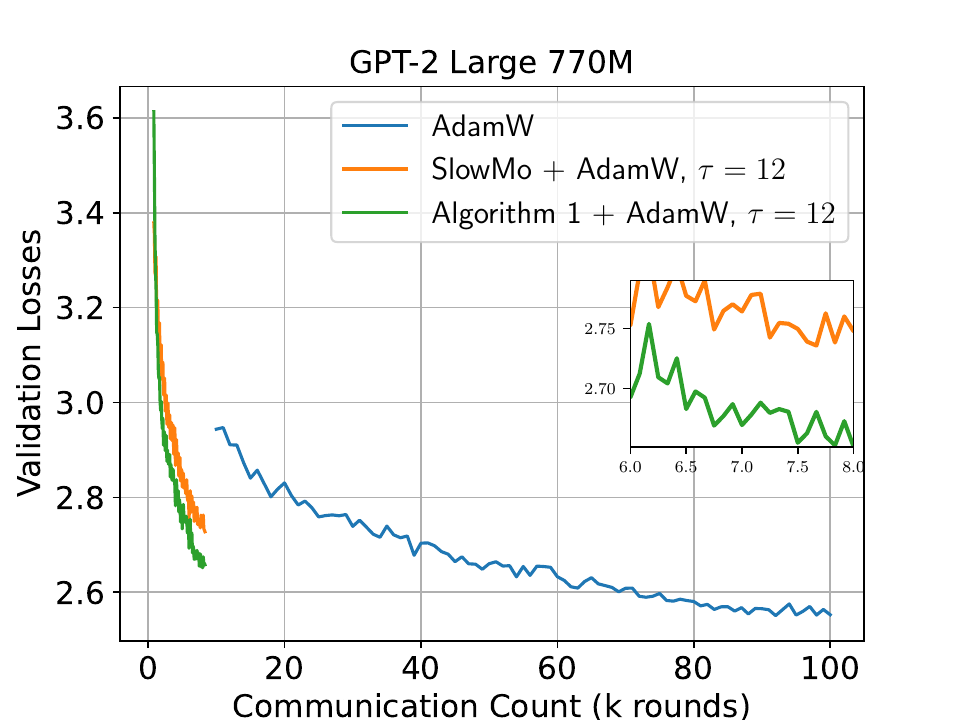}
    \end{minipage}
    \caption{Validation loss curves versus communication rounds for GPT-2 small, medium, and large. Communication interval for our Algorithm \ref{alg:dist-sign-mom} and SlowMo are set as  $\tau = 12$.}
    \label{fig:val-loss-curve-comm}
\end{figure*}

\begin{figure*}
    \centering
    \begin{minipage}{0.32\textwidth}
        \centering
        \includegraphics[width=\linewidth]{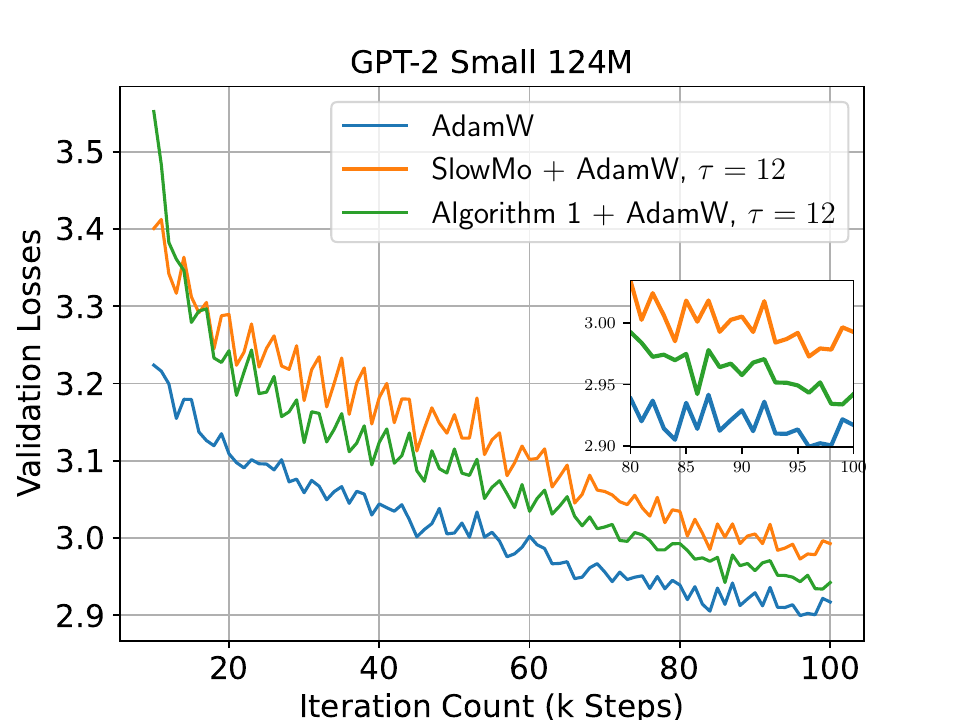}
    \end{minipage}\hfill
    \begin{minipage}{0.32\textwidth}
        \centering
        \includegraphics[width=\linewidth]{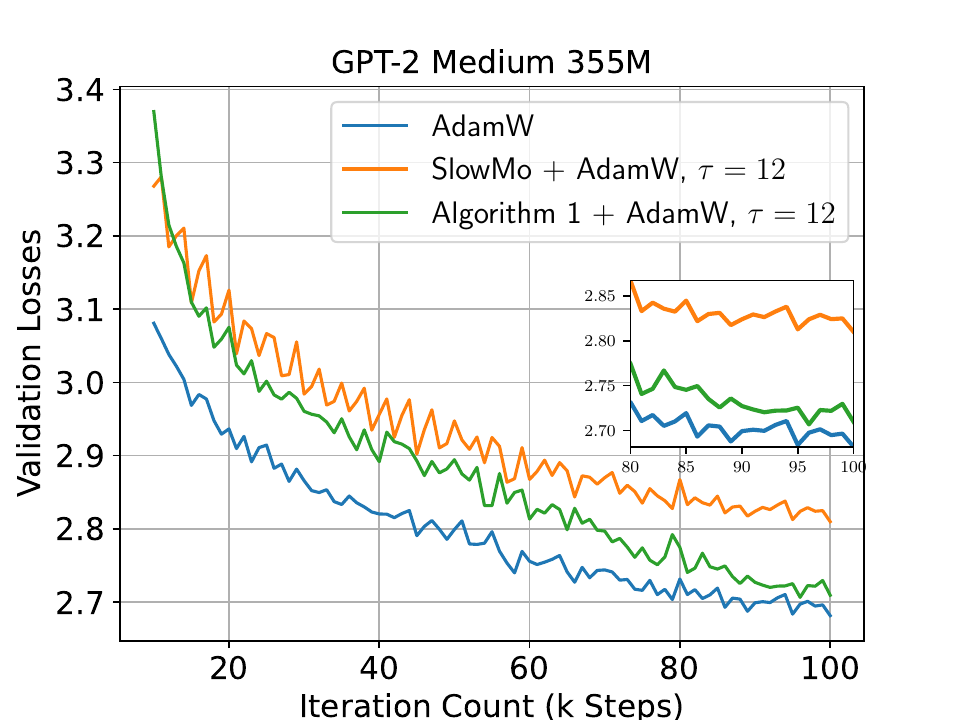}
    \end{minipage}\hfill
    \begin{minipage}{0.32\textwidth}
        \centering
        \includegraphics[width=\linewidth]{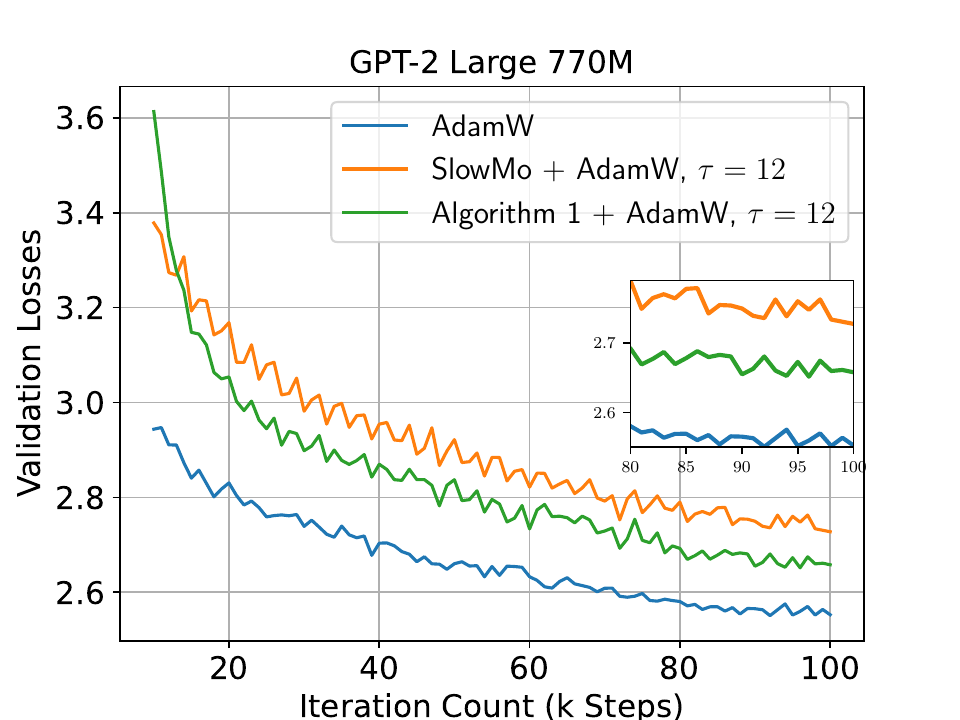}
    \end{minipage}
    \caption{Validation loss curves versus computation rounds for GPT-2 small, medium, and large. Communication interval for our Algorithm \ref{alg:dist-sign-mom} and SlowMo are set as  $\tau = 12$.}
    \label{fig:val-loss-curve}
\end{figure*}

\textit{Baselines.} We benchmark the empirical performance of our proposed framework, namely, Algorithm \ref{alg:dist-sign-mom}, using AdamW as the base optimizer, against the state-of-the-art distributed training method with multiple local steps, i.e., SlowMo (Algorithm \ref{alg:slowmo} in Appendix) that also uses AdamW as the base optimizer. AdamW is the dominant optimizer for the pre-training of Transformers \citep{liusophia}, and according to \citep{sun2024co2}, SlowMo with AdamW as the base optimizer is among the best performing distributed training methods with multiple local steps. We measure practical performance of Algorithm \ref{alg:dist-sign-mom} and SlowMo in terms of validation losses, i.e., token-level log perplexity. In general, under the same compute budget, due to significant communication savings, distributed methods with multiple local steps suffer from performance drop compared to those methods that communicate in every computation round. To demonstrate that our proposed Algorithm \ref{alg:dist-sign-mom} suffers from much less performance drop, we also report the validation losses of running standalone AdamW optimizer that communicates (all-reduce mini-batch gradients) in every computation round. We omit the comparison with local AdamW, i.e., averaging models over workers after every $\tau$ AdamW steps at each worker, due to its poor performance. 

\begin{table*}[h]
    \centering
    \renewcommand{\arraystretch}{1.3} 
    \begin{tabular}{|c|c||c|c||c|c||c|c|}
        \hline
        Alg. & Com. red. & \multicolumn{2}{c||}{GPT-2 Small 124M} & \multicolumn{2}{c||}{GPT-2 Medium 355M} & \multicolumn{2}{c|}{GPT-2 Large 770M} \\ 
        \hline
          &  & Val. &  Improv. & Val. &  Improv. & Val. &  Improv. \\
        \hline
        AdamW & N.A. & 2.917 &  & 2.682 &  & 2.552 &  \\
        \Xhline{1pt} 
        SlowMo & $12\times$ & 2.993 &  & 2.810 &  & 2.727 &  \\
        \hline
        Algorithm 1 & $12\times$ & 2.942 & \textbf{5.23\%} & 2.709 & \textbf{10.63\%} & 2.658 & \textbf{7.14\%} \\
        \Xhline{1pt}
        SlowMo & $24\times$ & 2.995 &  & 2.800 &  & 2.733 &  \\
        \hline
        Algorithm 1 & $24\times$ & 2.957 & \textbf{3.8\%} & 2.725 &  \textbf{7.79\%} & 2.864 & \textbf{14.00\%} \\
        \Xhline{1pt}
        SlowMo & $36\times$ & 3.000 &  & 2.818 &  & 2.736 &  \\
        \hline
        Algorithm 1 & $36\times$ & 2.972 & \textbf{2.84\%} & 2.750 & \textbf{7.04\%} & 2.714 & \textbf{2.22\%} \\
        \Xhline{1pt}
    \end{tabular}
    \caption{Performance comparison of standalone AdamW (with per iteration communcation), SlowMo + AdamW, and Algorithm 1 + AadmW under communcation intervals $\tau = 12, 24, 36$ for GPT-2 small, medium, and large. Com. red. means communication reduction, and Val. denotes final validation losses, Improv. is the improvement of Algorithm 1 over baseline SlowMo in terms of perplexity, i.e., the natural exponential of validation loss improvement.}
    \label{tab:example}
\end{table*}

\textit{Implementations.}  
We set the total batch size to 480 for all algorithms, and run each for 100k steps to ensure the same compute budget, as well as the same total number of tokens. In this setup, we fine-tune the recommended parameters for AdamW (both for the base optimizer AdamW in Algorithm \ref{alg:dist-sign-mom} and SlowMo): $\beta_1 = 0.9$, $\beta_2 = 0.95$, and a cosine learning rate schedule \citep{loshchilov2016sgdr} with a 2k-step warm-up. The peak learning rate (LR) is specified in Table \ref{tbe:gpt2-configs}, and the final LR is set to $0.05 \times$ the peak LR. For the global sign momentum step in Algorithm \ref{alg:dist-sign-mom}, we apply the recommended Lion parameters: $\beta_1 = 0.95$, $\beta_2 = 0.98$, and $\lambda = 0.1$ \citep{liusophia}. We tune the momentum coefficient and global LR for SlowMo, and the global LR for Algorithm \ref{alg:dist-sign-mom}. Performance is tested on three different sizes of GPT-2 models across varying numbers of workers.

\textit{Parameter tuning}. We performn the following model pre-training.
\begin{enumerate}
    \item  GPT-2 Small on 8 workers. We tune SlowMo with momentum coefficients in $\{0.2, 0.4, 0.5, 0.6, 0.7, 0.8, $ $0.9\}$, and global LR in $\{0.5,1, 2 \}$; and we tune global LR of Algorithm \ref{alg:dist-sign-mom} in $\{0.1,0.5, 0.8, 1.2, 1.5\}$.
    \item GPT-2 Medium on 8 workers. For SlowMo, we use the best performing momentum coefficient in GPT-2 Small training, and tune global LR in $\{0.5, 0.8, 1.0\}$. For Algorithm \ref{alg:dist-sign-mom}, we tune global LR in $\{0.5, 0.8, 1.0, 1.2, 1.5, 2.0\}$. 
    \item GPT-2 Large on 16 workers. SlowMo momentum coefficient is chosen from $\{0.5, 0.6\}$ and global LR is from $\{0.5, 1.0\}$. Algorithm \ref{alg:dist-sign-mom} uses the best global LR in $\{0.5, 0.8, 1.0, 1.2\}$. 
\end{enumerate}

\textit{Performance improvements}. In Figure \ref{fig:val-loss-curve-comm} and \ref{fig:val-loss-curve}, we present the validation loss curves against communication rounds and computation rounds, respectively, during the pre-training of three GPT-2 models. The communication interval $\tau$ for both our proposed Algorithm \ref{alg:dist-sign-mom} and SlowMo is set to 12. Figure \ref{fig:val-loss-curve-comm} and \ref{fig:val-loss-curve} show that our proposed Algorithm \ref{alg:dist-sign-mom} saves $12\times$ communications compared to standalone AdamW that communicates in every computation round, and achieves comparable performance. In all three experiments, our method consistently outperforms SlowMo by a significant margin, and suffers less than \textit{half performance drop} compared to the standalone AdamW, particularly for GPT-2 Small and GPT-2 Medium. For example, for GPT-2 medium, Algorithm 1 achieves final validation loss 2.709, while SlowMo only reaches 2.810. The 0.1 validation loss improvement is significant, since according to scaling laws \cite{liusophia, hoffmann2022training}, even an 0.05 validation loss improvement can take $2\times$ compute to achieve.

\textit{On the effects of communication interval $\tau$.} In Table \ref{tab:example}, we report the final validation losses of training three sized GPT-2 models with various algorithm setups. It shows that Algorithm \ref{alg:dist-sign-mom} achieves consistent and robust performance improvements  versus SlowMo under different communication intervals, $\tau = 12, 24, 36$, across all model sizes, highlighting the advantages of \textit{sign momentum} over standard momentum. Additionally, we observe that as $\tau$ increases, the performance gap between our method and SlowMo narrows. Nevertheless, a $10 \times$ or $20 \times$ reduction in communication would still offer substantial communication savings in practical distributed training scenarios where per-iteration communication is prohibitively expensive.

\textit{More base optimizers}. We also test our proposed framework Algorithm \ref{alg:dist-sign-mom} using base optimizer Sophia  \citep{liusophia}, a more recent optimizer shown to outperform AdamW. We train a GPT-2 small model over 4 workers, and compare with SlowMo that also uses Sophia base optimizer under communication interval $\tau = 12$, and standalone Sophia with per-iteration mini-batch gradient all-reduce. We use the suggested hyper parameters of Sophia, except tuning its peak LR and accordingly the minimum LR. Similar to previous experiments, we tune momentum coefficient and global LR for SlowMo and global LR for Algorithm \ref{alg:dist-sign-mom}. We report validation losses in Table \ref{tab:sophia}. Although both Algorithm \ref{alg:dist-sign-mom} and SlowMo achieve worse losses than the case of using AdamW as the base optimizer, Algorithm \ref{alg:dist-sign-mom} still improves over SlowMo over 5\%.  

\begin{table}[h]
    \centering
    \renewcommand{\arraystretch}{1.3} 
    \begin{tabular}{|c|c|c|c|}
        \hline
        Alg. & Com. red. & Val. & Improv. \\
        \hline
        Sophia & N.A. & 2.873 &  \\
        \hline
        SlowMo & $12\times$ & 3.016 &  \\
        \hline
        Algorithm 1 & $12\times$ & 2.965 & \textbf{5.23\%} \\
        \hline
    \end{tabular}
    \caption{Performance comparison of SlowMo and Algorithm \ref{alg:dist-sign-mom} with Sophia as the base optimizer.}
    \label{tab:sophia}
\end{table}


\subsection{Ablation studies} 
We conduct several ablation studies to explore the factors contributing to the effectiveness of sign momentum. In all these studies, we used the AdamW optimizer as the base optimizer, as described in the previous section.  

\textit{Lookahead and signed Lookahead optimizer}. We begin by testing whether momentum and signed momentum improve methods with multiple local steps in the case of a single worker. To do this, we configure Algorithm \ref{alg:dist-sign-mom} with $n = 1$, $\beta_1 = \beta_2 = \beta$, and $\lambda = 0$, reducing it to the signed Lookahead optimizer. Upon that, by replacing \eqref{eq:model-update} with $\x_{t + 1, 0} = \x_{t, 0} - \eta \gamma_t \vu_{t + 1}$, we recover the Lookahead optimizer \citep{zhang2019Lookahead}. We test the Lookahead optimizer using base optimizer AdamW in training GPT-2 medium, under $\tau = 48$, global LR = 1, with validation losses reported in Table \ref{tab:lookahead}. We also evaluate the signed Loookahead optimizer with AdamW as the base optimizer in training a GPT-2 small model with $\tau = 24$, global LR = 6, and the validations losses are reported in Table \ref{tab:signed-lookahead}. These results show that both Lookahead and signed Lookahead are effective in improving the performance of base optimizer when $n = 1$, highlighting the critical role of momentum in training Transformers.

\begin{table}[h]
    \centering
    \renewcommand{\arraystretch}{1.3} 
    \begin{tabular}{|c|c|c|c|}
        \hline
        Alg. & $\beta$ & Val. & Improv. \\
        \hline
        AadmW & N.A. & 2.682 &  \\
        \hline
        Lookahead & $0.1$ & 2.665 & \textbf{1.71\%} \\
        \hline
        Lookahead & $0.2$ & 2.660 & \textbf{2.22\%}  \\
        \hline
    \end{tabular}
    \caption{Performance comparison of AdamW and Lookahead with AadmW as the base optimizer.}
    \label{tab:lookahead}
\end{table} 

\begin{table}[h]
    \centering
    \renewcommand{\arraystretch}{1.3} 
    \begin{tabular}{|c|c|c|c|}
        \hline
        Alg. & $\beta$ & Val. & Improv. \\
        \hline
        AadmW &  N.A. & 2.917 &  \\
        \hline
        Signed Lookahead  &  $0.6$ & 2.896 & \textbf{2.12\%} \\
        \hline
        Signed Lookahead  &  $0.8$ & 2.890 & \textbf{2.74\%} \\
        \hline
    \end{tabular}
    \caption{Performance comparison of AdamW and signed Lookahead with AadmW as the base optimizer.}
    \label{tab:signed-lookahead}
\end{table}


\textit{Signed SlowMo}. We now consider the case with multiple workers, i.e., when $n > 1$. In this setting, we define signed SlowMo by setting $\beta_1 = \beta_2 = \beta$, and modifying the updates in \eqref{eq:ut-update} and \eqref{eq:model-update} as follows:
\begin{align*}
    \vu_{t + 1} & = \beta_1 \m_t + \frac{1 - \beta_1}{\gamma_t} \text{sign} (\x_{t, 0} - \x_{t, \tau}), \\ 
\x_{t + 1, 0} & = \x_{t, 0} - \eta \gamma_t \vu_{t + 1}.
\end{align*}
We use signed SlowMo to pre-train the GPT-2 Small model with a communication interval of $\tau = 12$, and report final validation losses in Table \ref{tab:signed-slowmo}. It shows that signed SlowMo achieves better final validation losses compared to SlowMo, showing the efficacy of incorporating sign momentum. However, it still underperforms compared to our proposed Algorithm \ref{alg:dist-sign-mom} (best loss 2.942), which we hypothesize is due to the acceleration effects provided by using $\beta_2 > \beta_1$.

\begin{table}[h]
    \centering
    \renewcommand{\arraystretch}{1.3} 
    \begin{tabular}{|c|c|c|c|}
        \hline
        Alg. & $\beta$ & Val. & Improv. \\
        \hline
        AadmW & N.A. & 2.917 &  \\
        \hline
        SlowMo & & 2.993 &  \\
        \hline
        Signed SlowMo & 0.5   & 2.974 & \textbf{1.92\%} \\
        \hline
        Signed SlowMo & 0.8   & 2.970 & \textbf{2.32\%} \\
        \hline
        Global AdamW & N.A  & 2.992 & \textbf{0.10\%} \\
        \hline
    \end{tabular}
    \caption{Performance comparison of AdamW;  SlowMo, signed SlowMo, and global AdamW with AadmW as the base optimizer.}
    \label{tab:signed-slowmo}
\end{table} 


\textit{Adaptive global update}. We further compare the global sign momentum method with a global step that mimics AdamW (see Appendix for a description). \citet{balles2018dissecting} show that Adam can be regarded as sign momentum with adaptive learning rate depending on relative variance. Its performance, as reported in Table \ref{tab:signed-slowmo}, is only comparable with SlowMo, partially indicating that this adaptivity brings limited benefits when used as the global step of local step methods. We defer its detailed description of this comparisons to the Appendix. 

\section{Conclusion and Discussions}
In this paper, we have introduced a new framework for distributed sign momentum with local steps, integrating sign momentum with a communication-efficient local step approach. The framework is flexible, accommodating a variety of off-the-shelf base optimizers. Comprehensive empirical results from pre-training GPT-2 models of various sizes from scratch consistently show improvements over the baseline method, validating the effectiveness of our approach in scenarios where communicating at every step is impractical. Additionally, a series of ablation studies highlights the impact of sign momentum in different scenarios. We have provided a general theoretical analysis for generic base optimizers, by utilizing a randomized version of the sign operator, and provided an $O(1/\sqrt{T})$ convergence rate when using SGD as the base optimizer. For SGD as base optimizer in particular, under actual sign operator, we have shown an optimal $O(1/T^{1/4})$ convergence that matches the rate of centralized method. To the best of our knowledge, this framework is a novel contribution to the literature.

Future directions include evaluating the performance of the proposed method on training vision models, such as Vision Transformers \citep{dosovitskiy2020image}, and investigating the convergence of Algorithm \ref{alg:dist-sign-mom} under broader parameter settings and with the real sign operator.


\newpage
\newpage
\bibliographystyle{plainnat}
\bibliography{refs}

\appendix
\onecolumn
\section{Optimizer descriptions}
\label{sec:optimizers}
In this section, we provide detailed descriptions of the algorithms referenced in the main text. In the absence of local steps, we present the pseudocode for Polyak’s Momentum (Algorithm \ref{alg:polyak-mom}), AdamW (Algorithm \ref{alg:adamw}), and Lion (Algorithm \ref{alg:lion}). We denote the stochastic gradient computed at \(\x\) using random samples \(\bxi\) as \(\nabla f(\x, \bxi)\), where \(\x\) and \(\bxi\) are defined based on the given context. For distributed algorithms with local steps, we give the pseducode of SlowMo (Algorithm \ref{alg:slowmo}) and Federated Majority Vote signSGD with Stochastic Simple Momentum (Federated MV-sto-signSGD-SIM in Algorithm \ref{alg:fedmv}). 
\begin{figure}[htbp]
    \centering
    \begin{minipage}{0.45\textwidth}
        \begin{algorithm}[H]
    \caption{AdamW, Adam \citep{kingma2014adam} with decoupled weight decay proposed by \citet{loshchilov2017decoupled}. Component-wise vector multiplication $\g_t^2 = \g_t \odot \g_t$, and $\sqrt{\widehat{\bv}_t}$ means component-wise square root.}
    \label{alg:adamw} 
    \begin{algorithmic}
    \REQUIRE Initialization $\x_0$, momentum coefficients $\beta_1, \beta_2$; weight decay $\lambda$;  learning rate $\eta$; $\epsilon = 10^{-8}$ 
    \STATE $\m_0 \gets \zero, \bv_0 \gets \zero$ 
    \FOR{\( t = 0 \) to \( T - 1 \)}
        \STATE Stochastic gradient: $\g_t \gets \nabla f(\x_t, \bxi_t)$
        \STATE Momentum updates: \\ 
        \(\qquad \m_{t + 1} \gets \beta_1 \m_{t} + ( 1- \beta_1) \g_t \) \\
        \( \qquad \bv_{t + 1} \gets \beta_2 \bv_{t} + (1 - \beta_2) \g_t^2  \)
        \STATE Bias-corrections: \\
        \(\qquad \widehat{\m}_{t + 1} \gets \m_{t + 1}/(1 - \beta_1^{t + 1}) \) \\
        \(\qquad \widehat{\bv}_{t + 1} \gets \bv_{t + 1}/(1 - \beta_2^{t +1}) \)
        \STATE Parameter update:  \\ 
        \(\qquad \x_{t + 1} \gets \x_t - \eta \Big( \frac{\widehat{\m}_{t + 1}}{\sqrt{\widehat{\bv}_{t + 1}} + \epsilon} + \lambda \x_{t}  \Big)\)
    \ENDFOR
  \STATE \textbf{return} \( \x_T \)
\end{algorithmic}
\end{algorithm} 
    \end{minipage}
    \hfill
    \begin{minipage}{0.45\textwidth}
    \begin{algorithm}[H]
    \begin{algorithmic}[1]
    \caption{Polyak's Momentum \citep{polyak1964some}.}
    \label{alg:polyak-mom}
    \REQUIRE Initialization $\x_0$, momentum coefficient $\beta$, learning rate $\eta > 0$
    \STATE \( \m_0 \gets \zero \)
    \FOR{$t = 0$ to $T - 1$}{
        \STATE Stochastic gradient: $\g_t \gets \nabla f(\x_t, \bxi_t)$
        \STATE Momentum update: 
        \( \m_{t+1} \gets \beta \m_t + \g_t\)
        \STATE Parameter update: $\x_{t+1} \gets \x_t - \eta \m_{t+1}$\;
    }
    \ENDFOR
    \STATE Return $\x_T$
    \end{algorithmic}
    \end{algorithm}
        \begin{algorithm}[H]
            \caption{Lion \citep{chen2024symbolic}. }
            \label{alg:lion}
            \begin{algorithmic}[1]
                \REQUIRE Initialization $\x_0$; momentum coefficients $\beta_1, \beta_2$; weight decay $\lambda$, learning rate $\eta$
                \STATE $\m_0 \gets \zero$
                \FOR{$t = 0$ to $T - 1$}
                    \STATE Stochastic gradient: \(\g_t \gets \nabla f(\x_t, \bxi_t) \)
                    \STATE Parameter update: \\
                    \( \qquad \vu_{t + 1} \gets \beta_1 \m_t + (1 - \beta_1) \g_t \) \\
                    \(\qquad \x_{t + 1} \gets \x_t - \eta( \text{sign}(\bu_{t + 1}) + \lambda \x_{t} ) \)
                    \STATE Momentum update: \\
                    \(\qquad \m_{t + 1} \gets \beta_2 \m_t + (1 - \beta_2) \g_t \)
                \ENDFOR
                \STATE Return $\x_T$
            \end{algorithmic}
        \end{algorithm}
    \end{minipage}
\end{figure}

\begin{figure}[htbp]
    \centering
    \begin{minipage}{0.45\textwidth}
        \begin{algorithm}[H]
    \caption{SlowMo, a general framework that combines local steps and global momentum step \citep{wang2019slowmo}. Worker $i$'s updating direction of base optimizer at iteration $(t, k)$ is denoted as $\vd_{t, i}^{(i)}$.}  
    \label{alg:slowmo}
    \begin{algorithmic}
    \REQUIRE Initialization $\x_{0, 0}$; base optimizer learning rate $\gamma_t$; inner loop steps $\tau$; global learning rate $\alpha$; initial momentum $\bu_0 = \zero$ 
    \FOR{\( t = 0 \) to \( T - 1 \) at worker $i$ in parallel}
        \FOR{\( k = 0 \) to \( \tau - 1\)}
            \STATE Base optimizer step: \(\x_{t, k + 1}^{(i)}  \gets \x_{t, k}^{(i)} - \gamma_t \vd_{t, k}^{(i)} \)
        \ENDFOR
        \STATE Exact averaging: \(\x_{t, \tau} \gets \frac{1}{n} \sum_{i=1}^n \x_{t, \tau}^{(i)}. \)
        \STATE Update slow momentum: \\
        \(\qquad \bu_{t + 1} \gets \beta \bu_t + \frac{1}{\gamma_t}(\x_{t, 0} - \x_{t, \tau}) \) 
        \STATE Update outer iterate: \(\x_{t + 1, 0} \gets \x_{t, 0} - \alpha \gamma_t \bu_{t + 1} \)
    \ENDFOR
    \STATE \textbf{return} \( \x_{T, 0} \)
    \end{algorithmic}
    \end{algorithm} 
    \end{minipage}
    \hfill
    \begin{minipage}{0.45\textwidth}
    \begin{algorithm}[H]
    \begin{algorithmic}[1]
    \caption{Federated MV-sto-signSGD-SIM \citep{sun2023momentum}. The randomized sign operator $\cS_r$ is as defined in \eqref{eq:rand-sign-1}.}
    \label{alg:fedmv}
    \REQUIRE Iterate initialization $\x_{-1} = \x_0 = \zero$, momentum initialization $\m_0 = \zero$, momentum coefficient $\beta \in [0, 1)$, local learning rate $\gamma$, inner loop steps $\tau$, outer momentum coefficient $\alpha$, global learning rate $\eta$, the uniform $\ell_2$ norm bound on stochastic gradient $B$
    \FOR{$t = 0$ to $T - 1$}{
        \STATE $\y_{t} \gets \x_t + \alpha(\x_t - \x_{t - 1})$
        \FOR{worker \(i = 1\) to \(n\) in parallel}
            \STATE \(\vz_{t, 0}^{(i)} \gets \y_t\)
            \FOR{\(k = 0\) to \(\tau - 1\)}
                \STATE \(\z_{t, k + 1}^{(i)} \gets \z_{t, k}^{(i)}  - \gamma \nabla f_i(\z_{t, k}^{(i)}, \bxi_{t, k}^{(i)})\)
            \ENDFOR
            \STATE \( \y_t^{(i)} \gets \z_{t, \tau}^{(i)}\)
            \STATE \( \m_{t + 1}^{(i)} \gets \beta \m_t^{(i)} + (1 - \beta) \nabla f_i(\y_t^{(i)}, \bxi_t^{(i)})  \)
        \ENDFOR
        \STATE Communicate and majority vote:\\
        \(\qquad \x_{t + 1} \gets \x_t - \eta \text{sign}\big( \sum_{i = 1}^n \cS_r(\m_i^{(t + 1)} \big)\)
    }
    \ENDFOR
    \STATE Return $\x_T$
    \end{algorithmic}
    \end{algorithm} 
    \end{minipage}
\end{figure} 


\section{Additional proofs}
\subsection{Proof of Theorem \ref{thm:rands-main}} 
\label{sec:proof-rand}
\subsubsection{Proof of Lemma \ref{lm:signopr}}
\begin{proof}
For $\cS_r$ defined in \eqref{eq:rand-sign-1} and \eqref{eq:rand-sign-2}, it can be readily verified that if $\| \bv \| \le B$, $\E_{\cS}[\cS(\bv)] = \bv/B$. Then, 
\begin{align*}
    \E_\cS[\| \cS_r(\bv) - \frac{\bv}{B} \|^2 ] = \E_\cS[\| \cS_r(\bv) - \E_\cS[\cS_r(\bv)] \|^2 ] = \E_\cS [\| \cS_r(\bv) \|^2] - \frac{\| \bv \|^2}{B^2} \le d, 
\end{align*}
where the last step follows from that $\cS_r(\bv)$'s component is either 0 or $\pm 1$. 
\end{proof}
\subsubsection{Virtual iterates}
Recall that at iteration $(t, k)$, the local updates are given by $\x_{t, k+1}^{(i)} = \x_{t, k}^{(i)} - \gamma \vd_{t, k}^{(i)}$. We define the global averages: 
\begin{align*}
    \vd_{t, k} := \frac{1}{n}\sum_{i = 1}^{n} \vd_{t, k}^{(i)}, \  \x_{t, k} := \frac{1}{n}\sum_{i = 1}^n \vx_{t, k}^{(i)} . 
\end{align*}
Then, one has
\begin{align*}
    \x_{t, k} = \x_{t,0} - \gamma \sum_{i = 0}^{k - 1} \bd_{t, i}. 
\end{align*}
In the global update step,
\begin{align*}
    \x_{t + 1, 0} = \x_{t, 0} - \eta \gamma \cS_r(\m_{t + 1} ) = \x_{t, 0} - \eta \gamma \cS_r \Big( \beta \m_t + (1 - \beta) \sum_{i = 0}^{\tau - 1} \vd_{t, i} \Big). 
\end{align*}
By Assumption \ref{as:bdd-sg} and Jensen's inequality, we have $\| \vd_{t, i} \| \le R$ almost surely (a.s.), and thus $\| \sum_{i = 0}^{\tau - 1} \vd_{t, i} \| \le \tau R$ a.s.. Further, $\m_{t + 1}$ is a convex combination of $\m_t$ and $\sum_{i = 0}^{\tau - 1} \vd_{t, i}$ with $\m_0 = \zero$, then one can readily show that $\| \m_t \| \le \tau R$ for all $t \ge 0$. Using $\cS_r$ with $B = \tau R$ and taking total expectation, we have 
\begin{align}
\begin{split}
\label{eq:x-update}
    \E[\x_{t + 1, 0}] 
    & = \E[\x_{t, 0}] - \eta \gamma \E[  \cS_r(\m_{t + 1}) ]  \\
    & = \E[\x_{t, 0}] - \frac{\eta \gamma}{\tau R} \E[\m_{t + 1}] \\ 
    & = \E[\x_{t, 0}] - \frac{\eta \gamma}{\tau R}\Big( \beta \E[\m_t] + (1 - \beta) \sum_{i = 0}^{\tau - 1} \E[ \vd_{t, i} ]   \Big)
\end{split}
\end{align}
We define the following auxiliary iterates $\y_{t, k}$ at iteration $(t, k)$ for $k = 0, \ldots, \tau$:
\begin{align}
\label{eq:vir-iter}
    \y_{t, k} := \x_{t, 0} - \frac{\eta \gamma}{\tau R ( 1 - \beta)} \big[ \beta \m_t + ( 1 - \beta ) \sum_{i = 0}^{k - 1} \vd_{t, i}  \big], 
\end{align}
where we use the convention $\sum_{i=0}^{-1} \vd_{t, i} = \zero$. From the definition in \eqref{eq:vir-iter}, we have the following relations for the virtual iterates $\{ \y_{t, k} \}$: 
\begin{align}
    \y_{t, k+ 1} = \y_{t, k} - \frac{\eta \gamma}{\tau R}   \vd_{t, k}, \forall k = 0, \ldots, \tau - 1, \label{eq:ynontau}. 
\end{align}
In addition, under total expectation one has, 
\begin{align}
\label{eq:yt-cancle}
\E \big[ \y_{t, \tau} \big] 
= \E[\x_{t, 0}] - \frac{\eta \gamma}{\tau R(1 - \beta)} \E[\m_{t + 1}] \overset{\eqref{eq:x-update}}{=} \E[\x_{t + 1, 0}] - \frac{\eta \gamma \beta}{\tau R(1 - \beta) } \E[\m_{t + 1}]  = \E[\y_{t + 1, 0}].
\end{align}

\subsubsection{Decent Lemma}
Let $\E_{t, k}$ denote the conditional expectation given history up to time $(t, k)$,
and $\E_{\cS} $ denote the expectation over the random sign operator $\cS_r$. Since each $f_i$ is $L$-smooth, the function $f$ is also $L$-smooth. We have for $k = 0, \ldots, \tau - 1$, 
\begin{align}
\label{eq:descent-lm}
    \E_{t, k}\big[ f(\y_{t, k + 1}) ] - f(\y_{t, k}) 
    & \le - \frac{\eta \gamma }{\tau R}\inpd{\nabla f(\y_{t, k}), \E_{t, k}[\vd_{t, k}]} + \frac{L\eta^2 \gamma^2 }{2 \tau^2 R^2 } \E_{t, k}\big[ \| \vd_{t, k} \|^2]. 
\end{align} 
For the first term on the right hand side, we have
\begin{align}
\begin{split}
    & -\inpd{\nabla f(\y_{t, k}), \E_{t, k}[\vd_{t, k}]}  \\
    = \ & -\inpd{\nabla f(\y_{t, k}) - \nabla f(\x_{t, k}), \E_{t, k}[\bd_{t, k}]} - \inpd{\nabla f(\x_{t, k}), \E_{t, k}[\bd_{t, k} ]} \\
    \le \ & \| \nabla f(\y_{t, k}) - \nabla f(\x_{t, k})\|^2 - \frac{1}{4} \| \E_{t,k }[\bd_{t, k}] \|^2 -\frac{1}{2} \| \nabla f(\x_{t, k} ) \|^2 + \frac{1}{2} \|\nabla f(\x_{t, k}) - \E_{t, k}[\bd_{t, k}] \|^2 \\
    \le \ & L^2 \| \y_{t, k} - \x_{t, k} \|^2 - \frac{1}{4} \| \E_{t,k }[\bd_{t, k}] \|^2 -\frac{1}{2} \| \nabla f(\x_{t, k} ) \|^2 + \frac{1}{2} \|\nabla f(\x_{t, k}) - \E_{t, k}[\bd_{t, k}] \|^2, \label{eq:cute-angle}
\end{split}
\end{align}
where the last inequality also follows from the $L$-smoothness of $f$. We first have
\begin{align}
\begin{split}
\label{eq:yxdiff}
\| \y_{t, k} - \x_{t, k} \|^2 
& = \| \gamma \big( 1 - \frac{\eta}{\tau R} \big) \sum_{i = 0}^{k-1} \bd_{t, i} - \frac{\eta \gamma \beta}{\tau R(1 - \beta)} \m_t \|^2 \\
& \le 2\gamma^2 \big(1 - \frac{\eta}{\tau R}\big)^2 \| \sum_{i=0}^{k-1} \vd_{t, i} \|^2 + \frac{2 \eta^2 \gamma^2 \beta^2}{\tau^2 R^2(1 - \beta)^2} \| \m_t \|^2.
\end{split}
\end{align}
We next bound the second term in \eqref{eq:descent-lm}.  
\begin{align}
\begin{split}
\label{eq:dtkdecom}
 \E_{t, k} [ \| \bd_{t, k} \|^2 ]  = \| \E_{t, k} [ \bd_{t, k} ] \|^2 + \E_{t, k}\big[ \| \bd_{t, k}  - \E_{t,k}[\bd_{t, k}] \|^2 \big]  
\end{split}
\end{align}
Let $\alpha := \frac{\eta \gamma}{\tau R}$. Combing \eqref{eq:descent-lm}-\eqref{eq:dtkdecom} leads to that
\begin{align*}
\E_{t, k}\big[ f(\y_{t, k + 1}) - f(\y_{t, k}) \big] 
& \le -\frac{\alpha}{2} \| \nabla f(\x_{t, k}| \|^2 - \frac{\alpha}{2}\big(\frac{1}{2}-\alpha L\big)\| \E_{t, k} \big[ \vd_{t, k} \big] \|^2 +  \frac{\alpha^2L}{2} \E_{t, k } \big[ \|\vd_{t, k} - \E_{t, k}[ \vd_{t, k} ]\|^2 \big] \\
& \quad + \frac{\alpha}{2} \| \nabla f(\x_{t, k}) - \E_{t, k}[\vd_{t, k}] \|^2  \\
& \quad +2\alpha \gamma^2 L^2\big( 1 - \frac{\eta}{\tau R} \big)^2 \| \sum_{i = 0}^{k - 1} \vd_{t, i} \|^2 + \frac{2\alpha^3 \beta^2 L^2}{(1 - \beta)^2} \| \m_t \|^2. 
\end{align*}
Taking the total expectation, 
\begin{align}
\begin{split}
\label{eq:descent-lm-decomp}
\E[f(\y_{t, k + 1}) - f(\y_{t, k})] 
& \le -\frac{\alpha}{2}\E[ \| \nabla f(\x_{t, k}) \|^2 ] - \frac{\alpha}{2}\big( \frac{1}{2} - \alpha L \big) \E \| \E_{t, k}[\vd_{t, k}] \|^2 + \frac{\alpha^2L}{2} \E \big[ \|\vd_{t, k} - \E_{t, k}[ \vd_{t, k} ]\|^2 \big] \\
& \quad + \frac{\alpha}{2} \E \big[ \| \nabla f(\x_{t, k}) - \E [\vd_{t, k}] \|^2 \big] \\
& \quad + \underbrace{2\alpha \gamma^2 L^2\big( 1 - \frac{\eta}{\tau R} \big)^2 \E \Big[ \| \sum_{i = 0}^{k - 1} \vd_{t, i} \|^2 \Big]}_{D_{t, k}} + \underbrace{\frac{2\alpha^3 \beta^2 L^2}{(1 - \beta)^2} \E \big[  \| \m_t \|^2 \big]}_{M_t}. 
\end{split}
\end{align}
We can now sum up \eqref{eq:descent-lm-decomp} to bound progress within in one round of local updates. Next, we try to bound progress across outer iterations. From the definition of $\y_{t, k}$ in \eqref{eq:vir-iter}, we have
\begin{align*}
    \y_{t + 1, 0} - \y_{t, \tau} = \x_{t + 1, 0} - \x_{t, 0}  + \frac{\eta \gamma}{\tau R} \m_{t + 1} = - \eta \gamma \cS_r(\m_{t + 1}) + \frac{\eta \gamma}{\tau R} \m_{t + 1}.  
\end{align*}
By the $L$-smoothness of $f$, we have
\begin{align}
\label{eq:cross-iter}
    f(\y_{t + 1, 0}) - f(\y_{t, \tau}) \le \inpd{\nabla f(\y_{t, \tau}), \y_{t + 1, 0} - \y_{t, \tau}} + \frac{L \eta^2  \gamma^2 }{2}\|\frac{1}{\tau R} \m_{t + 1} - \cS_r(\m_{ t + 1}) \|^2.  
\end{align}
Taking conditional expectation $\E_{t, \tau}$ (with randomness from $\cS_r$) on \eqref{eq:cross-iter}, and using \eqref{eq:yt-cancle} and Lemma \ref{lm:signopr}, we reach that
\begin{align}
\label{eq:cross-gap} 
    \E_{t, \tau} \big[ f(\y_{t + 1, 0}) - f(\y_{t, \tau}) \big] \le \frac{dL\eta^2 \gamma^2}{2}.
\end{align}
Using \eqref{eq:cross-gap}, we are now prepared to bound the progress across outer iterations. Taking the total expectation and summing \eqref{eq:descent-lm-decomp} and \eqref{eq:cross-gap} from $k = 0$ to $k = \tau - 1$ and from $t = 0$ to $t = T - 1$, we obtain:
\begin{align}
\begin{split}
\label{eq:descent-sum}
    \frac{\E[ f(\y_{T, 0})- f(\y_{0, 0})]}{\tau T}
    & \le -\frac{\alpha }{2\tau T} \sum_{ t= 0}^{T - 1}\sum_{k = 0}^{\tau - 1} \E \| \nabla f(\x_{t, k}) \|^2 - \frac{\alpha}{2\tau T}(\frac{1}{2} - \alpha L) \sum_{t = 0}^{T - 1} \sum_{k = 0}^{\tau - 1} \E [\| \E_{t, k}[\vd_{t, k}] \|^2 ] \\
    & \quad + \frac{1}{2}\alpha^2 \zeta^2 L +  \frac{dL\eta^2 \gamma^2}{2\tau} + \frac{\alpha}{2\tau T} \sum_{t = 0}^{T - 1}\sum_{k = 0}^{\tau - 1} \E [ \| \nabla f(\x_{t, k}) - \E[\vd_{t, k}] \|^2 ] \\
    & \quad + \frac{1}{\tau T} \sum_{t = 0}^{T - 1}\sum_{k = 0}^{\tau - 1} D_{t, k} + \frac{1}{T} \sum_{t = 0}^{T - 1} M_t. 
\end{split}
\end{align}
\subsubsection{Bounding $D_{t, k}$}
First, we have
\begin{align*}
    \| \sum_{i = 0}^{k - 1} \vd_{t, i} \|^2 \le 2 \| \sum_{i=0}^{k-1}(\vd_{t, i} - \E_{t, i}[\vd_{t, i}]) \|^2 + 2k\sum_{i=0}^{k - 1}\| \E_{t, i}[\vd_{t, i}]\|^2. 
\end{align*}
Then, taking total expectation and summing over $k = 0, \ldots, \tau -1 $, from the above relation we have
\begin{align}
\begin{split}
\label{eq:dti-run-sum}
    \sum_{k = 0}^{\tau - 1}\E\big[ \| \sum_{i = 0}^{k - 1} \vd_{t, i} \|^2 \big] 
    & \le 2 \sum_{k = 0}^{\tau - 1}  \E\big[ \| \sum_{i = 0}^{k - 1} ( \vd_{t, i} - \E_{t, i}[\vd_{t, i}]) \|^2 \big] + 2 \sum_{k = 0}^{\tau - 1}  k \sum_{i=0}^{k - 1} \E \big[ \| \E_{t, i}[\vd_{t, i}] \|^2 \big] \\
    & =  2 \sum_{k = 0}^{\tau - 1} \sum_{i= 0}^{k - 1} \E \big[ \| \vd_{t, i} - \E_{t, i}[\vd_{t, i}] \|^2 \big] + 2 \sum_{k = 0}^{\tau - 1}  k \sum_{i=0}^{k - 1} \E \big[ \| \E_{t, i}[\vd_{t, i}] \|^2 \big] \\
    & \le 2\sum_{k = 0}^{\tau - 1} k \zeta^2 + \tau (\tau - 1) \sum_{i = 0}^{k - 1} \E \big[ \| \E_{t, i}[\vd_{t, i}] \|^2 \big] \\
    & \le \tau(\tau - 1) \zeta^2 + \tau (\tau - 1) \sum_{i = 0}^{k - 1} \E \big[ \| \E_{t, i}[\vd_{t, i}] \|^2 \big]
\end{split}
\end{align}
where the first equality follows from the fact that the total expectation of crossing terms of $\| \sum_{i=0}^{k - 1} (\vd_{t, i} - \E_{t, i}[\vd_{t, i} ]) \|^2$ are zeros due to independence. From \eqref{eq:dti-run-sum} we have
\begin{align}
\label{eq:dsum-bd}
\frac{1}{\tau T} \sum_{t = 0}^{T - 1} \sum_{k = 0}^{\tau - 1} D_{t, k} \le 2\alpha \tau (\tau - 1 )\gamma^2 L^2 \big( 1 - \frac{\eta}{\tau R} \big)^2\Big( \frac{\zeta^2}{\tau} + \frac{1}{\tau T} \sum_{t = 0}^{T - 1} \sum_{k = 0}^{\tau - 1} \E[ \| \E_{t, k} [\vd_{t, k}] \|^2 ]\Big). 
\end{align}

\subsubsection{Bounding $M_t$}
By the definition of $\m_t$ and $\m_0 = \zero$, we have
\begin{align}
\begin{split}
\label{eq:mt-bound}
    \| \m_{t} \|^2 
    & = \| \sum_{s = 0}^{t - 1} \beta^{t - 1 - s} (1 - \beta) \Big( \sum_{k = 0}^{\tau - 1} \vd_{s, k} \Big) \|^2\\
    & \le 2(1 - \beta)^2 \| \sum_{s = 0}^{t - 1} \beta^{t - 1 - s}  \Big( \sum_{k = 0}^{\tau - 1} (\vd_{s, k} - \E_{s, k}[ \vd_{s, k} ] ) \Big) \|^2 + 2(1- \beta)^2 \| \sum_{s = 0}^{t - 1} \beta^{t - 1 - s}  \Big( \sum_{k = 0}^{\tau - 1} \E_{s, k} [\vd_{s, k}] \Big) \|^2.
\end{split}
\end{align}
We estimate the right two terms one by one. Taking total expectation and using the same arguments in equality of \eqref{eq:dti-run-sum}, 
\begin{align}
\begin{split}
\label{eq:mt-bound-p1}
& \E \Big[ \| \sum_{s = 0}^{t - 1} \beta^{t - 1 - s}  \Big( \sum_{k = 0}^{\tau - 1} (\vd_{s, k} - \E_{s, k}[ \vd_{s, k} ] ) \Big) \|^2 \Big] \\
= \ & \sum_{s = 0}^{t - 1} \beta^{2(t - 1 -s)} \E \Big[ \| \sum_{k = 0}^{\tau - 1} ( \vd_{s, k} - \E_{s, k}[\vd_{s, k}] \|^2 \Big] \\
= \ &  \sum_{s = 0}^{t - 1} \beta^{2(t - 1 -s)} \sum_{k = 0}^{\tau - 1} \E \big[ \| \vd_{s, k} - \E_{s, k}[\vd_{s, k}] \|^2 \big] \\
\le \ &  \tau \zeta^2 \sum_{s = 0}^{t- 1} \beta^{2(t - 1 -s)} \le \frac{\tau \zeta^2}{1 - \beta^2}.
\end{split}
\end{align}
Similarly, using Jensen's inequality, we have
\begin{align}
\begin{split}
\label{eq:mt-bound-p2}
    & \E \Big[ \| \sum_{s = 0}^{t - 1} \beta^{t - 1 - s}  \Big( \sum_{k = 0}^{\tau - 1} \E_{s, k} [\vd_{s, k}] \Big) \|^2 \Big ] \\ 
    \le \ &  \tau \Big( \sum_{s = 0}^{t - 1} \beta^{t - 1 -s} \Big) \sum_{s = 0}^{t - 1} \beta^{t - 1 - s} \sum_{k = 0}^{\tau - 1} \E \big[ \|  \E_{s, k} \vd_{s, k}\|^2 \big] \\
    \le \ & \frac{\tau}{1 - \beta } \sum_{s = 0}^{t - 1} \beta^{t - 1 - s} \sum_{k = 0}^{\tau - 1} \E \big[ \|  \E_{s, k} \vd_{s, k}\|^2 \big]. 
\end{split}
\end{align}
Combing \eqref{eq:mt-bound}-\eqref{eq:mt-bound-p2} gives that 
\begin{align}
\begin{split}
\label{eq:msum-bd}
\frac{ 1}{T}\sum_{t = 0}^{T - 1} M_t 
& \le \frac{4 \tau \alpha^3 \beta^2 \zeta^2 L^2}{1 - \beta^2} +  \frac{4 \tau \alpha^3 \beta^2 L^2}{1 - \beta}  \frac{1}{T} \sum_{t = 0}^{T - 1} \sum_{s = 0}^{ t  - 1} \beta^{ t- 1 - s}\sum_{k - 0}^{\tau - 1} \E\big[ \| \E_{s, k} [\vd_{s, k}]  \|^2 \big] \\
& = \frac{4 \tau \alpha^3 \beta^2 \zeta^2 L^2}{1 - \beta^2} +  \frac{4 \tau \alpha^3 \beta^2 L^2 }{1 - \beta}  \frac{1}{T} \sum_{t = 0}^{T - 2 }  \Big(  \sum_{s = 0}^{T - 2 - t} \beta^s  \Big)  \Big( \sum_{k = 0}^{\tau - 1} \E \big[ \| \E_{t,k} [\vd_{t, k} ] \|^2 \big] \Big) \\
& \le  \frac{4 \tau \alpha^3 \beta^2 \zeta^2 L^2}{1 - \beta^2} +  \frac{4 \tau \alpha^3 \beta^2 L^2}{(1 - \beta)^2} \frac{1}{T} \sum_{t = 0}^{T - 1} \sum_{k = 0}^{\tau - 1} \E \big[ \| \E_{t,k}[ \vd_{t, k} ] \|^2 \big]. 
\end{split}
\end{align}

\subsubsection{Convergence rate}
Putting \eqref{eq:dsum-bd} and \eqref{eq:msum-bd} back into \eqref{eq:descent-sum}, we obtain that
\begin{align}
\begin{split}
\label{eq:descent-clean}
    \frac{\E[ f(\y_{T, 0})- f(\y_{0, 0})]}{\tau T}
    & \le -\frac{\alpha }{2\tau T} \sum_{ t= 0}^{T - 1}\sum_{k = 0}^{\tau - 1} \E \| \nabla f(\x_{t, k}) \|^2 - \frac{\alpha}{2\tau T}(\frac{1}{2} - \alpha L) \sum_{t = 0}^{T - 1} \sum_{k = 0}^{\tau - 1} \E [\| \E_{t, k}[\vd_{t, k}] \|^2 ] \\
    & \quad + \frac{1}{2}\alpha^2 \zeta^2 L + \frac{dL\eta^2 \gamma^2}{2\tau} +  \frac{\alpha}{2\tau T} \sum_{t = 0}^{T - 1}\sum_{k = 0}^{\tau - 1} \E [ \| \nabla f(\x_{t, k}) - \E[\vd_{t, k}] \|^2 ] \\
    & \quad +2\alpha \tau (\tau - 1 )\gamma^2 L^2 \big( 1 - \frac{\eta}{\tau R} \big)^2\Big( \frac{\zeta^2}{\tau} + \frac{1}{\tau T} \sum_{t = 0}^{T - 1} \sum_{k = 0}^{\tau - 1} \E[ \| \E_{t, k} [\vd_{t, k}] \|^2 ]\Big)  \\ 
    & \quad + \frac{4 \tau \alpha^3 \beta^2 \zeta^2 L^2}{1 - \beta^2} +  \frac{4 \tau \alpha^3 \beta^2 L^2}{(1 - \beta)^2} \frac{1}{T} \sum_{t = 0}^{T - 1} \sum_{k = 0}^{\tau - 1} \E \big[ \| \E_{t,k}[ \vd_{t, k} ] \|^2 \big] \\
    & \le -\frac{\alpha }{2\tau T} \sum_{ t= 0}^{T - 1}\sum_{k = 0}^{\tau - 1} \E \| \nabla f(\x_{t, k} )\|^2 - \frac{\alpha C_1}{\tau T} \sum_{t =0}^{T- 1} \sum_{k = 0}^{\tau- 1}  \E \big[ \| \E_{t, k}[ \vd_{t, k} ] \|^2 \big] \\
    & \quad + \frac{\alpha}{2\tau T} \sum_{t = 0}^{T - 1}\sum_{k = 0}^{\tau - 1} \E [ \| \nabla f(\x_{t, k}) - \E[\vd_{t, k}] \|^2 ] + C_2,
\end{split}
\end{align}
where 
\begin{align*}
    C_1 & =  \frac{1}{4} - \frac{1}{2}\alpha L - 2\tau(\tau - 1) \gamma^2 L^2 \big( 1- \frac{\eta}{\tau R} \big)^2 - \frac{4 \tau^2 \alpha^2 \beta^2
 L^2}{( 1- \beta)^2 } , \\
 C_2 & = \alpha^2 \zeta^2 \Big[ \frac{L}{2} + 2 \alpha (\tau - 1)L^2\big( \frac{\tau R}{\eta} - 1\big)^2 +  \frac{4 \tau \alpha \beta^2 L^2}{1 - \beta^2} \Big] + \frac{d L \eta^2 \gamma^2}{2 \tau}. 
\end{align*}
When we take parameters to ensure that $C_1 \ge  0$, from \eqref{eq:descent-clean}, we have 
\begin{align}
\label{eq:cvg-2prune}
\frac{1}{\tau T} \sum_{ t= 0}^{T - 1}\sum_{k = 0}^{\tau - 1} \E \| \nabla f(\x_{t, k} )\|^2 \le \frac{2\E[ f(\y_{0, 0})- f(\y_{T, 0})]}{\alpha \tau T} + \frac{1}{\tau T} \sum_{t = 0}^{T - 1}\sum_{k = 0}^{\tau - 1} \E [ \| \nabla f(\x_{t, k}) - \E[\vd_{t, k}] \|^2 ] + \frac{2 C_2}{\alpha}. 
\end{align}
Note that $\y_{0, 0} = \x_{0, 0}$ since $\m_0 = \zero$, and we denote $f_* = \inf_{\x \in \R^d} f(\x)$. By taking $\alpha = \sqrt{\frac{n}{\tau T}}$ and recalling that $\alpha := \frac{\eta \gamma}{\tau R}$, we obtain that
\begin{align}
\begin{split}
\label{eq:cvg-rate-rs}
\frac{1}{\tau T} \sum_{ t= 0}^{T - 1}\sum_{k = 0}^{\tau - 1} \E \| \nabla f(\x_{t, k} )\|^2 \le \frac{2(f(\x_{0, 0})- f_*)}{ \sqrt{n  \tau T}} + \frac{1}{\tau T} \sum_{t = 0}^{T - 1}\sum_{k = 0}^{\tau - 1} \E [ \| \nabla f(\x_{t, k}) - \E[\vd_{t, k}] \|^2 ] \\ + \zeta^2 L \sqrt{ \frac{n}{ \tau T }} +  \frac{4n L^2 \zeta^2}{T}\Big[ \big(\frac{\tau R}{\eta} - 1 \big)^2 + \frac{2 \beta^2}{1 - \beta^2} \Big] + d L R^2 \sqrt{ \frac{n \tau}{T} }. 
\end{split} 
\end{align}
Now we return to the constraint of $C_1$ when $\alpha = \sqrt{\frac{n}{\tau T}}$. Using $\alpha = \frac{\eta \gamma}{\tau R}$ in $C_1$ and rearranging terms gives that 
\begin{align*}
     \frac{1}{4} - \frac{1}{2}\sqrt{ \frac{n L^2}{\tau T}} - \frac{nL^2}{\tau T}\Big[ 2 \tau (\tau -1)\big( \frac{\tau R}{\eta} - 1 \big)^2 + \frac{4\tau^2 \beta^2}{(1 - \beta)^2} \Big] \ge 0. 
    \end{align*}
The above relation reduces to 
\begin{align*}
\frac{\tau T}{n L^2} - 2\sqrt{\frac{\tau T}{n L^2}} + 1 & \ge 8\tau(\tau - 1)\big( \frac{\tau R}{\eta} - 1 \big)^2 + \frac{16 \tau^2 \beta^2}{(1 - \beta)^2} + 1, \\
\sqrt{\frac{\tau T}{n L^2}} &  \ge  \Big[  8\tau(\tau - 1)\big( \frac{\tau R}{\eta} - 1 \big)^2 + \frac{16 \tau^2 \beta^2}{(1 - \beta)^2} + 1 \Big]^{1/2}  + 1 . 
\end{align*}
Since
\begin{align*}
    \Big[ \Big[  8\tau(\tau - 1)\big( \frac{\tau R}{\eta} - 1 \big)^2 + \frac{16 \tau^2 \beta^2}{(1 - \beta)^2} + 1 \Big]^{1/2} + 1 \Big]^2 \le 16 \tau(\tau - 1)\big( \frac{\tau R}{\eta} - 1 \big)^2 + \frac{32 \tau^2 \beta^2}{(1 - \beta)^2} + 4,
\end{align*}
it suffices to have
\begin{align*}
    T \ge 4nL^2 \Big[ 4(\tau - 1)\big( \frac{\tau R}{\eta} - 1 \big)^2 + \frac{8\tau \beta^2}{(1 - \beta)^2} + \frac{1}{\tau} \Big]. 
\end{align*}

\subsection{Proof of Theorem \ref{thm:base-sgd}} 
\begin{proof}
Since $\bd_{t, k}^{(i)} = \nabla f_i(\x_{t, k}^{(i)}, \bxi_{t, k}^{(i)})$, we have $\E_{t, k}  [ \vd_{t, k} ] = \frac{1}{n}\sum_{i=1}^n \nabla f_i(\x_{t, k}^{(i)})$. Then,
\begin{align}
\label{eq:drift-bounds}
    \| \nabla f(\x_{t, k}) - \E_{t, k}[\vd_{t, k}] \|^2 = \| \frac{1}{n}\sum_{i = 1}^n \nabla f_i(\x_{t, k}) - \frac{1}{n}\sum_{i = 1}^n \nabla f_i(\x_{t,k}^{(i)}) \|^2  \le \frac{L^2}{n}\sum_{i = 1}^n \| \x_{t, k} - \x_{t, k}^{(i)} \|^2, 
\end{align}
where in the last inequality we use Jensen's inequality and $L$-smoothness of $f_i$'s. The quantity in \eqref{eq:drift-bounds} is a crucial in analyzing local SGD updates \citep{khaled2020tighter}. Since $\x_{t, 0}^{(i)}$ is synchronized for every $t$, the local steps in Algorithm \ref{alg:dist-sign-mom} is equivalent to that in local SGD. We 
can use the existing bound for deviations of local iterates in \citet{wang2019slowmo} (equation (87)) or \cite{yu2019linear} (Lemma 5 with $\beta = 0$) in \eqref{eq:drift-bounds}, then the term of effect of base optimizer in Theorem \ref{thm:rands-main} can be bounded as
\begin{align}
\label{eq:effect-base-localsgd}
    \frac{1}{\tau T} \sum_{t = 0}^{T - 1} \sum_{k = 0}^{\tau - 1} \E \Big[ \| \nabla f(\x_{t, k}) - \E_{t, k}[\vd_{t, k}] \|^2 \Big] \le  \frac{L^2}{n \tau T} \sum_{t = 0}^{T - 1} \sum_{k = 0}^{\tau - 1} \sum_{i = 1}^n \E \big[ \| \x_{t, k} - \x_{t, k}^{(i)} \|^2 \big] \le \frac{2\gamma^2 L^2( \sigma^2 \tau  + 3\delta^2 \tau^2)}{1 - 12 \gamma^2 L^2 \tau^2}. 
\end{align}
When $\gamma L \tau \le 1/6$, we have $1/(1 - 12\gamma^2 L^2 \tau^2) \le 3/2$. Combining this bound with \eqref{eq:effect-base-localsgd}, and using $\gamma = \frac{R}{\eta}\sqrt{\frac{n \tau}{T} }$, the bound in Theorem \ref{thm:rands-main} reduces to 
\begin{align*}
\frac{1}{\tau T} \sum_{ t= 0}^{T - 1}\sum_{k = 0}^{\tau - 1} \E \| \nabla f(\x_{t, k} )\|^2 \le \frac{2(f(\x_{0, 0})- f_*)}{ \sqrt{n  \tau T}} + \frac{3 n \tau^2 L^2 R^2 (\sigma^2 + 3 \tau \delta^2)}{\eta^2 T} \\
+ \sigma^2 L \sqrt{ \frac{n}{ \tau T }}  + \frac{4n L^2 \sigma^2}{T}\Big[ \big(\frac{\tau R}{\eta} - 1 \big)^2 + \frac{2 \beta^2}{1 - \beta^2} \Big] + dLR^2\sqrt{\frac{n \tau}{T}}. 
\end{align*}
When $T \ge 36 n L^2 R^2 \tau^3/\eta^2$, we have $\gamma L \tau \le 1/6$, and thus $\frac{1}{\tau T} \sum_{ t= 0}^{T - 1}\sum_{k = 0}^{\tau - 1} \E \| \nabla f(\x_{t, k} )\|^2 = O(1/\sqrt{T})$.
\end{proof}

\subsection{Proof of Theorem \ref{thm:hard_sign}} 
\subsubsection{Preliminaries}
Since we us SGD as the base optimizer, $\bd_{t, k}^{(i)} = \nabla f_i(\x_{t, k}^{(i)}, \bxi_{t, k}^{(i)})$. We define the following notations: 
\begin{align}
\tnab g(\x_{t, 0}) & = 
\frac{1}{n}\sum_{i=1}^n \sum_{k = 0}^{\tau - 1} \nabla f_i(\x_{t, k}^{(i)}, \bxi_{t, k}^{(i)}),  \\
\tnab f(\x_{t, 0}) & = \frac{1}{n}\sum_{i = 1}^n \sum_{k = 0}^{\tau - 1}  \nabla f_i(\x_{t, k}^{(i)}).
\end{align}
We define the accumulative stochastic gradient noise at $t$-th outer iteration as: 
\begin{align*}
    \bdta_t & = \tnab g(\x_{t, 0}) - \tnab f(\x_{t, 0}).
\end{align*}
We use
\begin{align*}
    \beps_{t + 1} & = \m_{t + 1} - \tau  \nabla f(\x_{t, 0}),
\end{align*}
to denote the gradient estimation error. Then, with $\gamma_t = \gamma$ for all $t$ an $\beta_1 = \beta_2 = \beta$, one has the following outer iteration for $t = 0, \ldots, T - 1$: 
\begin{align}
    \m_{t + 1} & = \beta \m_{t} + (1 - \beta) \tnab g(\x_{t, 0}), \\
    \x_{t + 1, 0} & = \x_{t, 0} - \eta \gamma \sign(\m_{t + 1}). \label{eq:gene-sign-update}
\end{align} 
We initialize
\begin{align*}
    & \m_0 = \zero_d, \\
    & \x_{0, 0} =\x_{-1, 0} - \eta \gamma  \sign(\m_0). 
\end{align*}
Hence $\x_{0, 0} = \x_{-1, 0}$ and $\beps_0 = -\tau \nabla f(\x_{t, 0})$. 

\subsubsection{Descent Lemma}
\begin{lemma}
\label{lm:sign-decent}
Under Assumption \ref{as:smooth}. For the iterates $\{\x_{t, 0}\}$ in Algorithm \ref{alg:dist-sign-mom} with $\beta_1 = \beta_2 = \beta, \lambda = 0$, it satisfies that 
\begin{align}
\label{eq:hs-descent-lm}
    \frac{1}{T}\sum_{t = 0}^{T  - 1} \E \big[ \| \nabla f(\x_{t, 0}) \ \|_1 \big] \le  \frac{f(\x_{0, 0}) - f_*}{\eta \gamma T} + \frac{2 \sqrt{d}}{\tau T} \sum_{t = 0}^{T - 1} \E \big[ \|  \beps_{t + 1} \| \big] + \frac{\eta \gamma d L}{2}.  
\end{align}
\end{lemma} 
\begin{proof}
The proof of this Lemma is adapted from \cite{sun2023momentum}. From Assumption \ref{as:smooth} we have
\begin{align*}
    f(\x_{t + 1, 0}) 
    & \le f(\x_{t, 0}) + \nabla f(\x_{t, 0})^\top (\x_{t +1, 0} - \x_{t, 0}) + \frac{L}{2}\| \x_{t + 1, 0} - \x_{t, 0} \|^2 \\
    & \underset{\eqref{eq:gene-sign-update}}{\le} f(\x_{t, 0}) - \eta \gamma \nabla f(\x_{t, 0})^\top \sign(\m_{t + 1} ) + \frac{\eta^2 \gamma^2 dL}{2} \\
    & \le f(\x_{t, 0} ) - \eta \gamma \nabla f(\x_{t, 0})^\top \sign(\nabla f(\x_{t, 0})) + \eta \gamma \nabla f(\x_{t, 0})^\top( \sign(\nabla f(\x_{t, 0}) - \sign(\m_{t + 1})) + \frac{\eta^2 \gamma^2 d L}{2} \\
    & = f(\x_{t, 0}) - \eta \gamma \| \nabla f(\x_{t, 0}) \|_1 + \eta \gamma \sum_{r=1}^d [\nabla f(\x_{t, 0})]_r \cdot [\sign([\nabla f(\x_{t, 0})]_r) - \sign([\m_{t + 1}]_r)] + \frac{\eta^2 \gamma^2 d L}{2}.
\end{align*}
Recall that $\beps_{t + 1} = \m_{t + 1} - \tau \nabla f(\x_{t, 0})$. If $\sign([\nabla f(\x_{t,0})]_r) \ne \sign ([\m_{t + 1}]_r)$, then 
\begin{align*}
    | [\beps_{ t + 1}]_r | = |[\m_{t + 1}]_r - \tau [\nabla f(\x_{t})]_r | \ge \tau | [\nabla f(\x_{t, 0})]_r |,  
\end{align*}
and thus
\begin{align*}
    [\nabla f(\x_{t, 0})]_r \cdot [\sign([\nabla f(\x_{t, 0})]_r) - \sign([\m_{t + 1}]_r)] \le \frac{2 | [\beps_{t + 1}]_r |}{\tau}.
\end{align*}
It follows that
\begin{align*}
    f(\x_{t + 1, 0}) \le f(\x_{t, 0}) - \eta \gamma \| \nabla f(\x_{t, 0}) \|_1 + \frac{2 \eta \gamma \sqrt{d} \| \beps_{t + 1} \|}{\tau} + \frac{\eta^2 \gamma^2 d L}{2}.  
\end{align*}
Summing over $t = 0, \ldots, T-1$ and taking expectation, we obtain the desired relation in \eqref{eq:hs-descent-lm}. 
\end{proof}

\subsubsection{Bounding estimation errors}
We next try to bound the estimation error $\E[\| \beps_{t + 1} \|]$. 
By the recursion of $\m_t$ and the definition of $\bdta_t$, for any $0 \le t \le T - 1$, 
\begin{align*} 
    \m_{t + 1} 
    & = \beta \m_t + (1 - \beta) \tnab g(\x_{t, 0}) \\
    & = \beta (\beps_t + \tau \nabla f(\x_{t - 1, 0})) + (1 - \beta)(\bdta_t + \tau \nabla f(\x_{t, 0})) + (1 - \beta)(\tnab f(\x_{t, 0}) - \tau \nabla f(\x_{t, 0}))
\end{align*}
Then, subtracting $\tau \nabla f(\x_{t, 0})$ from both sides of the above relation gives that
\begin{align*}
    \beps_{t + 1} 
    & = \beta \beps_t + \beta \tau (\nabla f(\x_{t - 1, 0}) - \nabla f(\x_{t, 0})) + (1 - \beta) \bdta_t + (1 - \beta)(\tnab f(\x_{t, 0}) - \tau \nabla f(\x_{t, 0})) \\
    & = \beta^{t + 1} \beps_0 + \tau \sum_{s=0}^t \beta^{t + 1 - s } (\nabla f(\x_{s - 1, 0}) - \nabla f(\x_{s, 0})) + (1-\beta)\sum_{s =0}^t \beta^{t - s} \bdta_s \\
    & \quad  + (1 - \beta) \sum_{s = 0}^t \beta^{t - s}( \tnab f(\x_{s, 0}) - \tau \nabla f(\x_{s, 0}) ). 
\end{align*} 
Taking Euclidean norm and total expectation on the above relation leads to
\begin{align}
\label{eq:bounding_et}
\begin{split}
    \E [ \| \beps_{t + 1} \| ] 
    & \le \beta^{t + 1}  \E [\| \beps_0 \|] + \underbrace{\tau \sum_{s=0}^t \beta^{t + 1 - s } \E \big[ \| \nabla f(\x_{s - 1, 0}) - \nabla f(\x_{s, 0})\| \big]}_{E_1} + \underbrace{(1-\beta) \E \big[ \| \sum_{s = 0}^t \beta^{t - s} \bdta_s \| \big]}_{E_2}  \\
    & \quad  + \underbrace{(1 - \beta) \E \big[ \|  \sum_{s = 0}^t \beta^{t - s}( \tnab f(\x_{s, 0}) - \tau \nabla f(\x_{s, 0}) ) \| \big]}_{E_3}.
\end{split}
\end{align}
We next bound $E_1, E_2, E_3$ one by one. 

\textbf{Bounding $E_1$}. From $L$-smoothness of $f$, we have, for all $s = 0, \ldots, t$, 
\begin{align*}
\| \nabla f(\x_{s - 1, 0}) - \nabla f(\x_{s ,0}) \| \le L \|\x_{s - 1, 0} - \x_{s, 0} \| \le \eta \gamma \sqrt{d}L.
\end{align*}
Thus, 
\begin{align}
\label{eq:e1_bd}
    E_1 \le \tau \eta \gamma \sqrt{d} L \sum_{s = 0}^t \beta^{t + 1 - s} \le \frac{\tau \eta \gamma \sqrt{d} L}{1 - \beta}.
\end{align}

\textbf{Bounding $E_2$}. By definition, 
\begin{align}
    E_2 & = (1 - \beta)  \E \big[ \| \sum_{s = 0}^t \beta^{t - s} \bdta_s \| \big] \notag \\ 
    & \le (1 - \beta) \Big( \E [ \| \sum_{s= 0}^t \beta^{t - s} \bdta_s \|^2 ] \Big)^{1/2} \notag \\
    & \le (1 - \beta) \Big( \E [ \| \sum_{s = 0}^t \beta^{t - s} ( \frac{1}{n}\sum_{i=1}^n \sum_{k = 0}^{\tau - 1} \nabla f_i(\x_{s, k}^{(i)}, \bxi_{s, k}^{(i)}) - \frac{1}{n}\sum_{i = 1}^n \sum_{k = 0}^{\tau - 1}  \nabla f_i(\x_{s, k }^{(i)}) ) \|^2 ] \Big)^{1/2} \notag \\
    & \le (1 - \beta) \Big( \frac{1}{n^2}  \sum_{s = 0}^t \beta^{2t - 2s} ( \sum_{i=1}^n \sum_{k = 0}^{\tau - 1} \E [ \| \nabla f_i(\x_{s, k}^{(i)}, \bxi_{s, k}^{(i)}) -  \nabla f_i(\x_{s, k }^{(i)}) ) \|^2 ] \Big)^{1/2} \label{eq:tower_exp} \\
    & \le \sqrt{\frac{\tau (1 - \beta) }{n}} \sigma, \label{eq:beta_ineq}
 \end{align}
where we used the mutual independence of stochastic gradient noise over workers, and the low of total expectations in \eqref{eq:tower_exp}, and used $\sqrt{1 - \beta^2} > \sqrt{1 - \beta}$ in \eqref{eq:beta_ineq}. 

\textbf{Bounding $E_3$}. By Assumption \ref{as:smooth} and \ref{as:bdd-sg}, we have 
\begin{align}
\label{eq:e3_bound}
\begin{split}
    E_3 & = (1 - \beta)  \E \big[ \|  \sum_{s = 0}^t \beta^{t - s}\big[ \frac{1}{n}\sum_{i = 1}^n \sum_{k = 0}^{\tau - 1}\big( \nabla f_i(\x_{s,k}^{(i)}) -\nabla f_i(\x_{s, 0}) \big) \big] \| \big] \\
    & \le (1 - \beta) \frac{L}{n} \sum_{s= 0}^t \beta^{t - s} \sum_{i = 1}^n \sum_{k = 0}^{\tau - 1} \E [ \| \x_{s, k}^{(i)} - \x_{s, 0} \| ] \\
    & \le \frac{1}{2}(1 - \beta) \tau (\tau - 1) \eta L R \sum_{s = 0}^t \beta^{t - s} \\
    & \le \frac{1}{2}\tau^2 \eta L R, 
\end{split}
\end{align}
where we used $\|\x_{s, k}^{(i)} - \x_{s, 0}\| \le \eta k R$ from the boundedness of $\| \nabla f_i(\x_{s, k}^{(i)}, \bxi_{s, k}^{(i)})\| \le R$. In light of \eqref{eq:e1_bd}\eqref{eq:beta_ineq}\eqref{eq:e3_bound}, the relation \eqref{eq:bounding_et} leads to
\begin{align}
\label{eq:et_finalbd}
    \E[ \| \beps_{t + 1} \|] \le \tau\beta^{t +1 } \| \nabla f(\x_{0, 0}) \| + \frac{\tau \eta \gamma \sqrt{d} L}{1 - \beta} + \sigma \sqrt{\frac{\tau (1  - \beta)}{n}} + \frac{1}{2}\tau^2 \eta L R.
\end{align}

\subsubsection{Proof of Theorem \ref{thm:hard_sign}}
We are now in a position  to show Theorem \ref{thm:hard_sign}.
\begin{proof}
    Using Lemma \ref{lm:sign-decent} with bound of $\E[ \| \beps_{t +1} \|]$ in \eqref{eq:et_finalbd}, we obtain
    \begin{align*}
        \frac{1}{T}\sum_{t = 0}^{T  - 1} \E \big[ \| \nabla f(\x_{t, 0}) \ \|_1 \big] \le  
        \frac{f(\x_{0, 0}) - f_*}{\eta \gamma T} + \frac{2\sqrt{d}}{T(1- \beta)} \|\nabla f(\x_{0, 0}) \| + \frac{2dL\eta \gamma}{1 - \beta}  + 2\sigma \sqrt{\frac{(1 - \beta) d}{\tau n} } + \eta L  \big( \tau R \sqrt{d} + \frac{\gamma d}{2} \big).
    \end{align*} 
Taking $\eta = \frac{1}{L T^{3/4}}$, and $1 - \beta = \frac{1}{\sqrt{T}}$, it holds that
\begin{align*}
     \frac{1}{T}\sum_{t = 0}^{T  - 1} \E \big[ \| \nabla f(\x_{t, 0}) \ \|_1 \big] \le \frac{L(f(\x_{0, 0}) - f_*)}{\gamma  T^{1/4}} + \frac{2\sqrt{d}}{T^{1/2}} \|\nabla f(\x_{0, 0}) \| + \frac{2d \gamma}{T^{1/4}}  + \frac{2\sigma}{T^{1/4}} \sqrt{\frac{d}{\tau n} } +  \frac{\sqrt{d} \tau R + \gamma d / 2}{T^{3/4}}.
\end{align*}
\end{proof}

\section{Additional experiment details}
\subsection{Implementations}
Our implementations\footnote{https://github.com/shuhuayu/dist-sign-momentum} build on codebases nanoGPT\footnote{https://github.com/karpathy/nanoGPT}, FairScale\footnote{https://github.com/facebookresearch/fairscale}, and Sophia\footnote{https://github.com/Liuhong99/Sophia},  where we use nanoGPT for GPT-2 \citep{radford2019language} modeling, Sophia for the optimizer Sophia, and FairScale for distributed training modules. Our algorithm is implemented in the FairScale distributed training modules. 


\subsection{Additional experiment results}

\textit{Comparison with local AdamW.} In our experiments, we omit comparisons with Local AdamW, which uses AdamW as the base optimizer for local steps and periodically averages parameters across workers. In Figure \ref{fig:compare-local}, we compare the performance of Local AdamW with SlowMo and our proposed method under communication intervals of $\tau = 12$ and $\tau = 24$. The results show that Local AdamW is significantly slower than both SlowMo and our method, consistent with the findings of \citet{sun2024co2}.
    \begin{figure*}[htbp]
    \centering
    \begin{minipage}{0.32\textwidth}
        \centering
        \includegraphics[width=\linewidth]{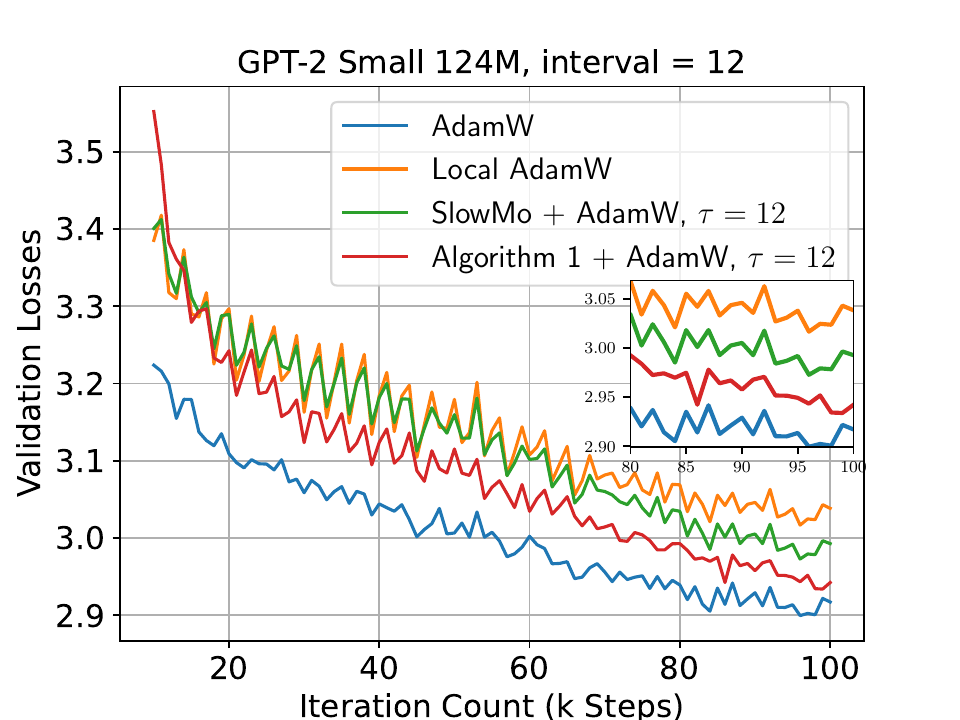}
    \end{minipage}
    \begin{minipage}{0.32\textwidth}
        \centering
        \includegraphics[width=\linewidth]{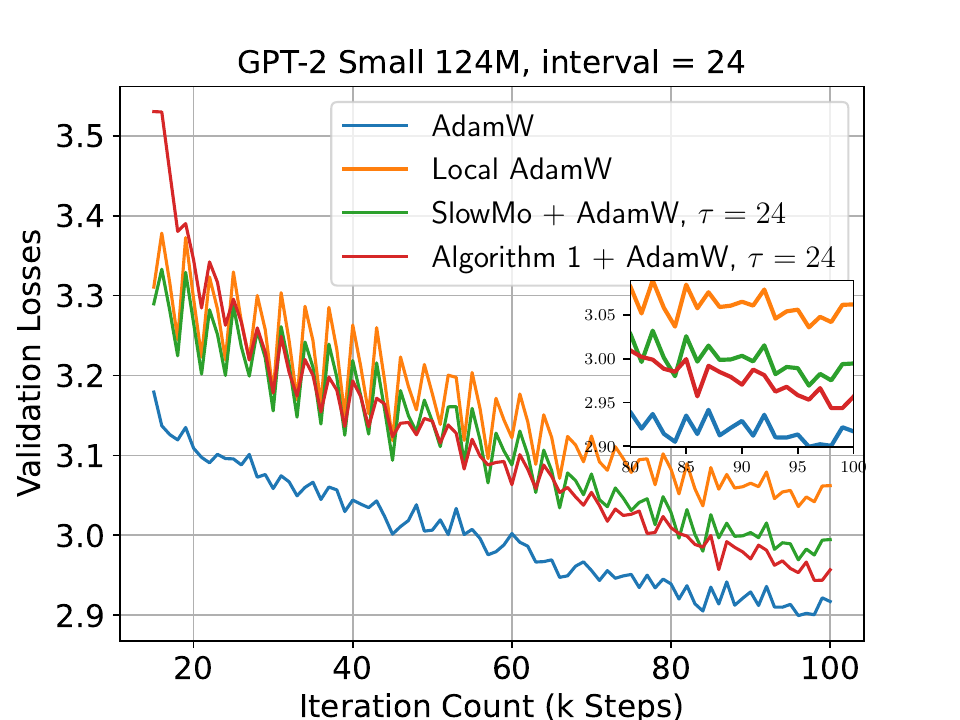}
    \end{minipage}
    \caption{Validation loss curves for communication interval $\tau = 12, 24.$}
    \label{fig:compare-local}
\end{figure*}

\textit{Comparisons of optimization errors}. In addition to the improved generalization performance in terms of validation loss, we evaluate the optimization errors of our proposed method by plotting the training losses in Figure \ref{fig:compare-train-12} for a communication interval of $\tau = 12$. The results in Figure \ref{fig:compare-train-12} demonstrate that our method consistently outperforms SlowMo in terms of optimization errors. 
\begin{figure*}[htbp]
    \centering
    \begin{minipage}{0.32\textwidth}
        \centering
        \includegraphics[width=\linewidth]{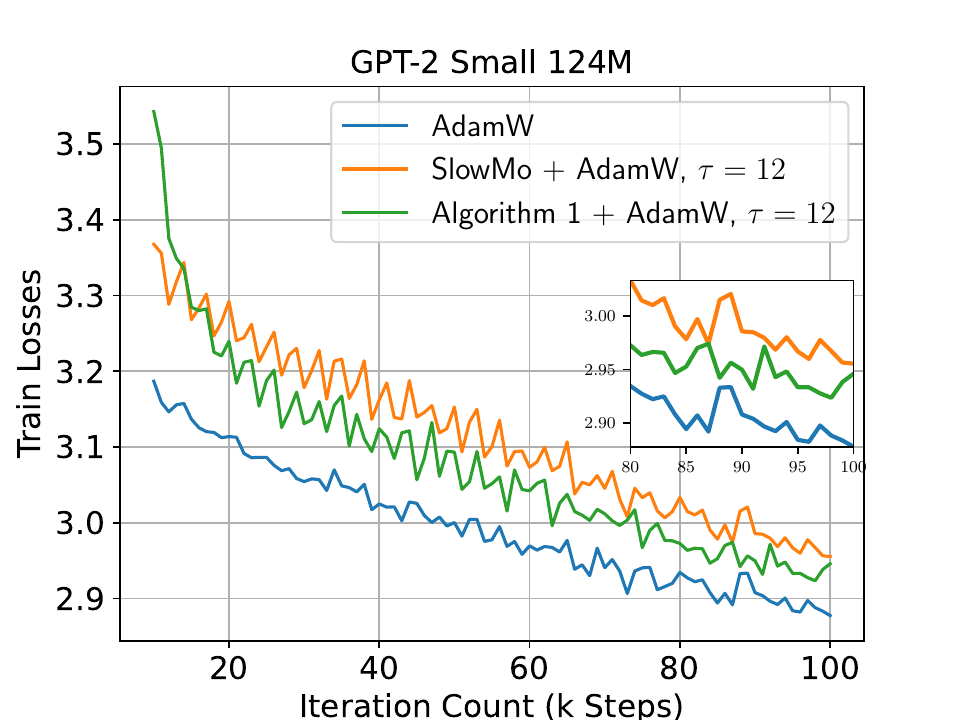}
    \end{minipage}\hfill
    \begin{minipage}{0.32\textwidth}
        \centering
        \includegraphics[width=\linewidth]{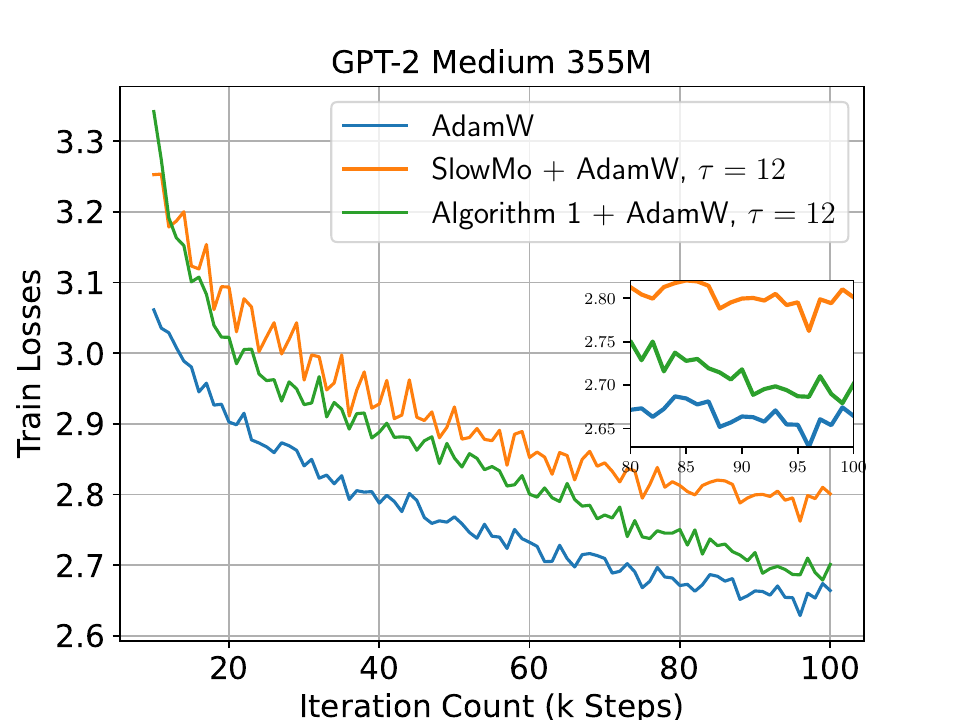}
    \end{minipage}\hfill
    \begin{minipage}{0.32\textwidth}
        \centering
        \includegraphics[width=\linewidth]{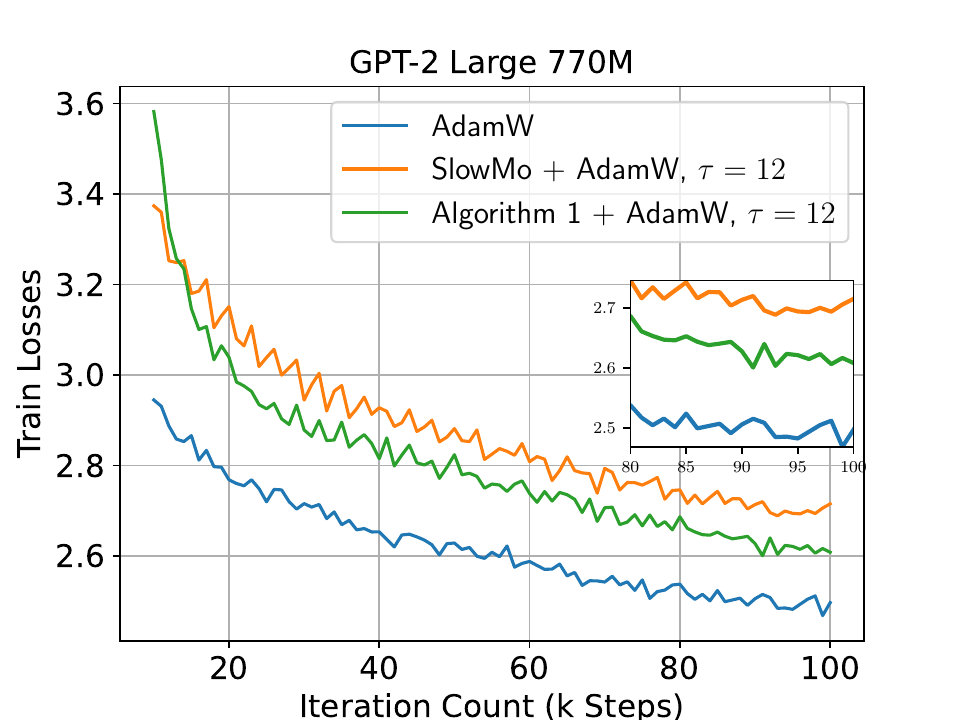}
    \end{minipage}
    \caption{Training loss curves for communication interval $\tau = 12$.}
    \label{fig:compare-train-12}
\end{figure*} 

\textit{Additional validation loss curves.} We plot the validation loss curves in Figures \ref{fig:val-24} for communication intervals $\tau = 24$.  
\begin{figure*}[htbp]
    \centering
    \begin{minipage}{0.32\textwidth}
        \centering
        \includegraphics[width=\linewidth]{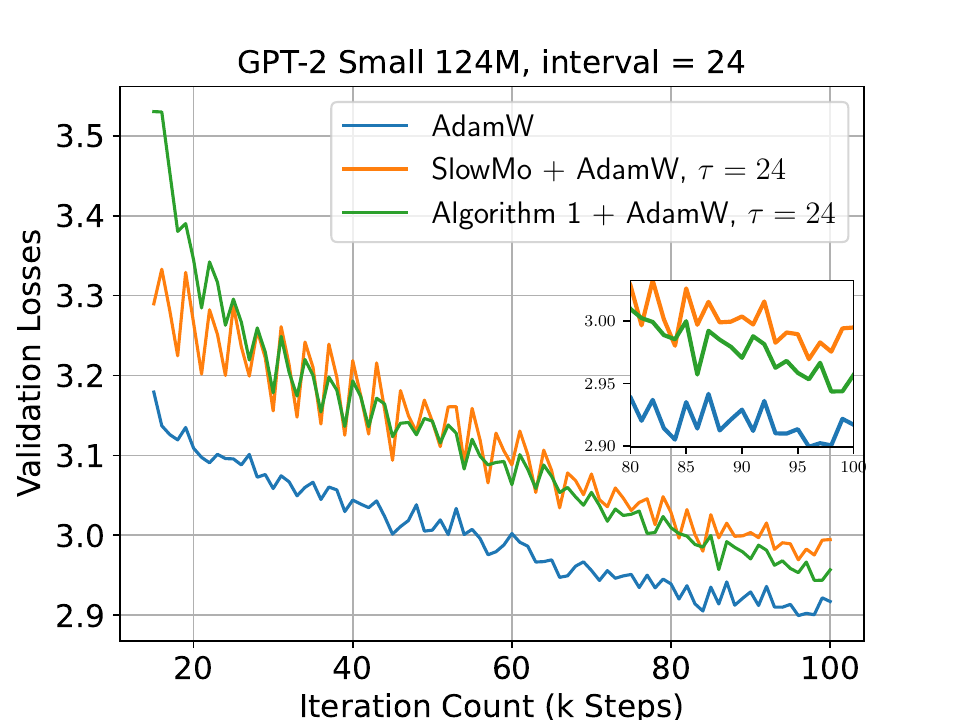}
    \end{minipage}\hfill
    \begin{minipage}{0.32\textwidth}
        \centering
        \includegraphics[width=\linewidth]{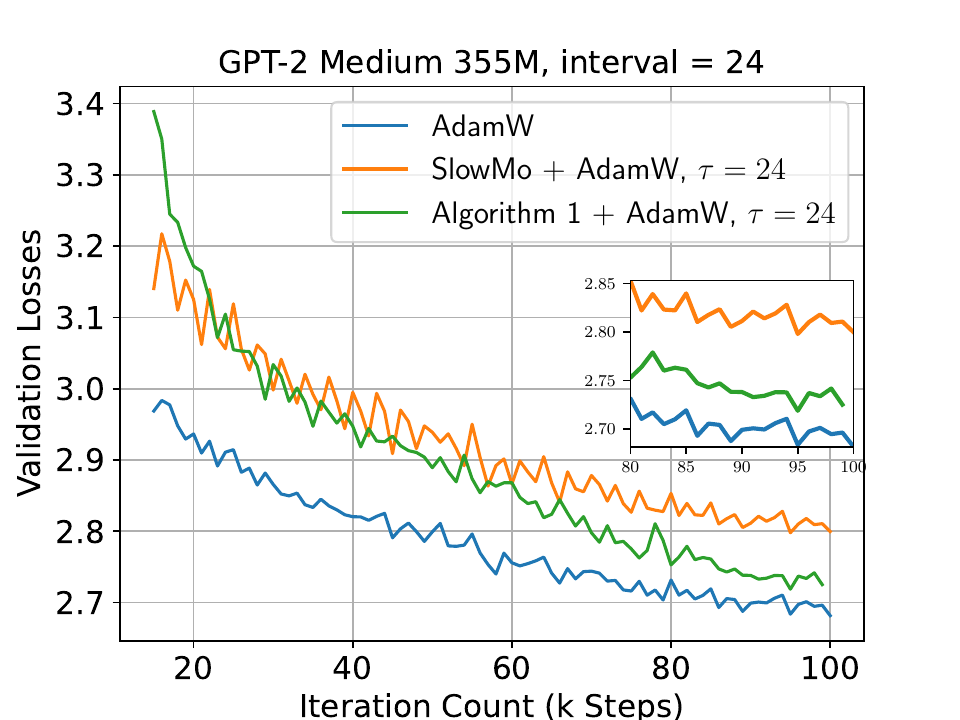}
    \end{minipage}\hfill
    \begin{minipage}{0.32\textwidth}
        \centering
        \includegraphics[width=\linewidth]{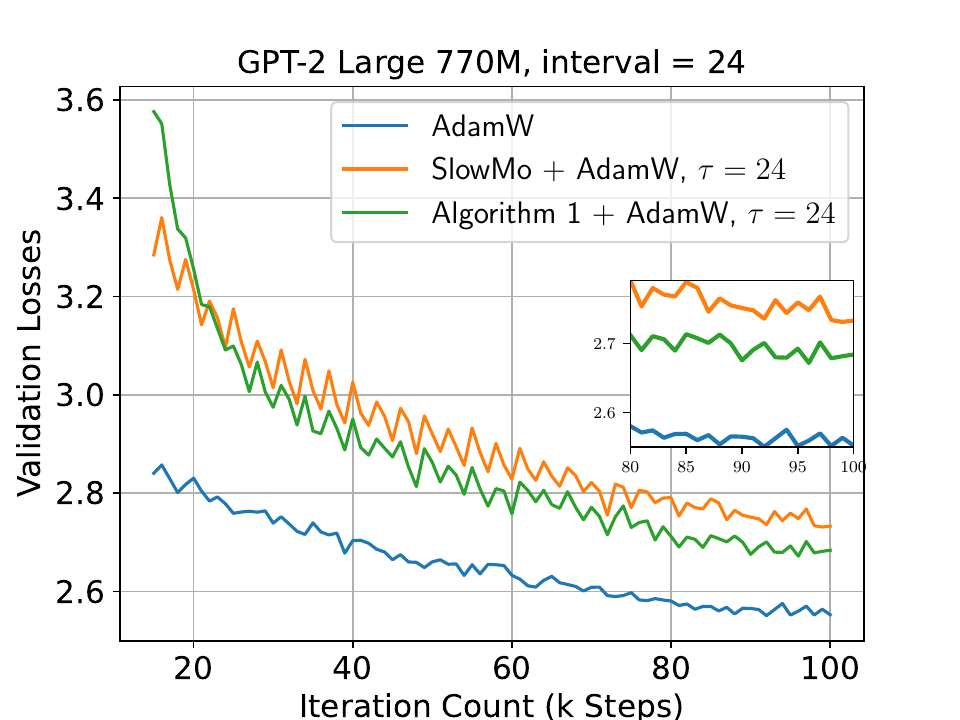}
    \end{minipage}
    \caption{Validation loss curves for communication interval $\tau = 24$.}
    \label{fig:val-24}
\end{figure*} 

\textit{Details of global AdamW step}. The detailed description of a global update that mimics AdamW is presented in Algorithm \ref{alg:global-adamw}. 

\begin{algorithm}[ht]
\caption{Global AdamW with local steps. Component-wise vector multiplication $\g_t^2 = \g_t \odot \g_t$, and $\sqrt{\widehat{\bv}_t}$ means component-wise square root.}
\label{alg:global-adamw}
\begin{algorithmic}[1]
    \REQUIRE Initialization $\x_{0, 0}$; local update directions $\vd_{t, k}^{(i)}$;  local learning rate $\gamma$; global learning rate $\eta$; momentum coefficients $\beta_1, \beta_2$; $\epsilon = 10^{-8}$.   
    \STATE \( \m_t \gets \zero, \bv_t \gets \zero \)
    \FOR{\( t = 0 \) to \( T - 1 \)}
        \FOR{\(i = 1 \) to \( n \) in parallel}
            \FOR{ \( k = 0 \) to \( \tau - 1\)}
            \STATE \( \x^{(i)}_{t, k+1} \gets \x^{(i)}_{t, k} - \gamma_t \vd_{t, k}^{(i)} \)
        \ENDFOR
        \ENDFOR
        \STATE All-reduce step: $\x_{t, \tau} = \frac{1}{n} \sum_{i=1}^n \x^{(i)}_{t, \tau}$
        \STATE Pseudo-gradient: \( \g_t \gets \frac{1}{\gamma_t}(\x_{t, 0} - \x_{t, \tau})\)
        \STATE Momentum updates: \\ 
        \(\qquad \m_{t + 1} \gets \beta_1 \m_{t} + ( 1- \beta_1) \g_t \) \\
        \( \qquad \bv_{t + 1} \gets \beta_2 \bv_{t} + (1 - \beta_2) \g_t^2  \)
        \STATE Bias-corrections: \\
        \(\qquad \widehat{\m}_{t + 1} \gets \m_{t + 1}/(1 - \beta_1^{t + 1}) \) \\
        \(\qquad \widehat{\bv}_{t + 1} \gets \bv_{t + 1}/(1 - \beta_2^{t +1}) \)
        \STATE Parameter update:  \\ 
        \(\qquad \x_{t + 1, 0} \gets \x_{t, 0} - \eta \Big( \frac{\widehat{\m}_{t + 1}}{\sqrt{\widehat{\bv}_{t + 1}} + \epsilon} + \lambda \x_{t, 0}  \Big)\)
        \STATE Synchronize over all workers: \\
        \( \qquad \forall i \in [n], \ \x^{(i)}_{t + 1, 0} \gets \x_{t+ 1, 0} \)
    \ENDFOR
  \STATE \textbf{return} \( \x_{T, 0} \)
\end{algorithmic}
\end{algorithm} 

\textit{On hyper-parameter fine-tuning}. For AdamW, we fine-tune the peak LR and use the recommended values for $(\beta_1, \beta_2)$ and weight decay $\lambda$ from \citet{liusophia}. Our results are consistent with those reported in \citet{liusophia}. The same $(\beta_1, \beta_2)$ and $\lambda$ values are used for the base optimizer AdamW in both SlowMo and our proposed method. For SlowMo, we fine-tune the local peak LR, global LR, and momentum coefficients. In our proposed method, we adopt the values for $(\beta_1, \beta_2)$ and $\lambda$ suggested for Lion in \citet{liusophia}, while fine-tuning only the global LR. Note that our proposed method is more sensitive to communication intervals due to its signed update nature, whereas in SlowMo, the momentum buffer can adapt to the accumulation of gradients from any number of local steps $\tau$, making it less sensitive to communication intervals.

\end{document}

%% file: macros.tex
\def\R{\mathbb{R}}
\def\E{\mathbb{E}}

\def\sign{\mathrm{sign}}

\newcommand{\bd}{\boldsymbol{d}}

\newcommand{\g}{\boldsymbol{g}}

\newcommand{\m}{\boldsymbol{m}}

\newcommand{\bu}{\boldsymbol{u}}
\newcommand{\bv}{\boldsymbol{v}}

\newcommand{\x}{\boldsymbol{x}}
\newcommand{\y}{\boldsymbol{y}}
\newcommand{\z}{\boldsymbol{z}}

\def\cD{\mathcal{D}}

\def\cS{\mathcal{S}}

\newcommand{\tnab}{\tilde{\nabla}}

\newcommand{\bxi}{\boldsymbol{\xi}}
\newcommand{\beps}{\boldsymbol{\epsilon}}
\newcommand{\bdta}{\boldsymbol{\delta}}

\newcommand{\zero}{\mathbf{0}}

\newcommand{\inpd}[1]{\langle #1 \rangle}

\newcommand{\minimize}{\mathop{\mathrm{minimize}}}



%% file: math_commands.tex
\usepackage{amsmath, amsfonts, bm}

\def\1{\bm{1}}

\def\vd{{\bm{d}}}

\def\vu{{\bm{u}}}
\def\vv{{\bm{v}}}

\def\vx{{\bm{x}}}

\def\vz{{\bm{z}}}

\DeclareMathAlphabet{\mathsfit}{\encodingdefault}{\sfdefault}{m}{sl}
\SetMathAlphabet{\mathsfit}{bold}{\encodingdefault}{\sfdefault}{bx}{n}

















%% file: arxiv.bbl
\begin{thebibliography}{46}
\providecommand{\natexlab}[1]{#1}
\providecommand{\url}[1]{\texttt{#1}}
\expandafter\ifx\csname urlstyle\endcsname\relax
  \providecommand{\doi}[1]{doi: #1}\else
  \providecommand{\doi}{doi: \begingroup \urlstyle{rm}\Url}\fi

\bibitem[Assran et~al.(2019)Assran, Loizou, Ballas, and
  Rabbat]{assran2019stochastic}
Mahmoud Assran, Nicolas Loizou, Nicolas Ballas, and Mike Rabbat.
\newblock Stochastic gradient push for distributed deep learning.
\newblock In \emph{International Conference on Machine Learning}, pages
  344--353. PMLR, 2019.

\bibitem[Balles and Hennig(2018)]{balles2018dissecting}
Lukas Balles and Philipp Hennig.
\newblock Dissecting adam: The sign, magnitude and variance of stochastic
  gradients.
\newblock In \emph{International Conference on Machine Learning}, pages
  404--413. PMLR, 2018.

\bibitem[Bernstein et~al.(2018)Bernstein, Wang, Azizzadenesheli, and
  Anandkumar]{bernstein2018signsgd}
Jeremy Bernstein, Yu-Xiang Wang, Kamyar Azizzadenesheli, and Animashree
  Anandkumar.
\newblock signsgd: Compressed optimisation for non-convex problems.
\newblock In \emph{International Conference on Machine Learning}, pages
  560--569. PMLR, 2018.

\bibitem[Bernstein et~al.(2019)Bernstein, Zhao, Azizzadenesheli, and
  Anandkumar]{bernsteinsignsgd}
Jeremy Bernstein, Jiawei Zhao, Kamyar Azizzadenesheli, and Anima Anandkumar.
\newblock signsgd with majority vote is communication efficient and fault
  tolerant.
\newblock In \emph{International Conference on Learning Representations}, 2019.

\bibitem[Chen and Huo(2016)]{chen2016scalable}
Kai Chen and Qiang Huo.
\newblock Scalable training of deep learning machines by incremental block
  training with intra-block parallel optimization and blockwise model-update
  filtering.
\newblock In \emph{2016 IEEE International Conference on Acoustics, Speech and
  Signal Processing (ICASSP)}, pages 5880--5884. IEEE, 2016.

\bibitem[Chen et~al.(2024{\natexlab{a}})Chen, Liu, Liang, et~al.]{chenlion}
Lizhang Chen, Bo~Liu, Kaizhao Liang, et~al.
\newblock Lion secretly solves a constrained optimization: As lyapunov
  predicts.
\newblock In \emph{The Twelfth International Conference on Learning
  Representations}, 2024{\natexlab{a}}.

\bibitem[Chen et~al.(2024{\natexlab{b}})Chen, Liang, Huang, Real, Wang, Pham,
  Dong, Luong, Hsieh, Lu, et~al.]{chen2024symbolic}
Xiangning Chen, Chen Liang, Da~Huang, Esteban Real, Kaiyuan Wang, Hieu Pham,
  Xuanyi Dong, Thang Luong, Cho-Jui Hsieh, Yifeng Lu, et~al.
\newblock Symbolic discovery of optimization algorithms.
\newblock \emph{Advances in neural information processing systems}, 36,
  2024{\natexlab{b}}.

\bibitem[Dean et~al.(2012)Dean, Corrado, Monga, Chen, Devin, Mao, Ranzato,
  Senior, Tucker, Yang, et~al.]{dean2012large}
Jeffrey Dean, Greg Corrado, Rajat Monga, Kai Chen, Matthieu Devin, Mark Mao,
  Marc'aurelio Ranzato, Andrew Senior, Paul Tucker, Ke~Yang, et~al.
\newblock Large scale distributed deep networks.
\newblock \emph{Advances in neural information processing systems}, 25, 2012.

\bibitem[Dosovitskiy(2020)]{dosovitskiy2020image}
Alexey Dosovitskiy.
\newblock An image is worth 16x16 words: Transformers for image recognition at
  scale.
\newblock \emph{arXiv preprint arXiv:2010.11929}, 2020.

\bibitem[Gokaslan et~al.(2019)Gokaslan, Cohen, Pavlick, and
  Tellex]{gokaslan2019openwebtext}
Aaron Gokaslan, Vanya Cohen, Ellie Pavlick, and Stefanie Tellex.
\newblock Openwebtext corpus.
\newblock 2019.

\bibitem[Gorbunov et~al.(2021{\natexlab{a}})Gorbunov, Burlachenko, Li, and
  Richt{\'a}rik]{gorbunov2021marina}
Eduard Gorbunov, Konstantin~P Burlachenko, Zhize Li, and Peter Richt{\'a}rik.
\newblock Marina: Faster non-convex distributed learning with compression.
\newblock In \emph{International Conference on Machine Learning}, pages
  3788--3798. PMLR, 2021{\natexlab{a}}.

\bibitem[Gorbunov et~al.(2021{\natexlab{b}})Gorbunov, Hanzely, and
  Richt{\'a}rik]{gorbunov2021local}
Eduard Gorbunov, Filip Hanzely, and Peter Richt{\'a}rik.
\newblock Local sgd: Unified theory and new efficient methods.
\newblock In \emph{International Conference on Artificial Intelligence and
  Statistics}, pages 3556--3564. PMLR, 2021{\natexlab{b}}.

\bibitem[Hoffmann et~al.(2022)Hoffmann, Borgeaud, Mensch, Buchatskaya, Cai,
  Rutherford, de~Las~Casas, Hendricks, Welbl, Clark,
  et~al.]{hoffmann2022training}
Jordan Hoffmann, Sebastian Borgeaud, Arthur Mensch, Elena Buchatskaya, Trevor
  Cai, Eliza Rutherford, Diego de~Las~Casas, Lisa~Anne Hendricks, Johannes
  Welbl, Aidan Clark, et~al.
\newblock Training compute-optimal large language models. arxiv.
\newblock \emph{arXiv preprint arXiv:2203.15556}, 2022.

\bibitem[Jiang et~al.(2017)Jiang, Balu, Hegde, and
  Sarkar]{jiang2017collaborative}
Zhanhong Jiang, Aditya Balu, Chinmay Hegde, and Soumik Sarkar.
\newblock Collaborative deep learning in fixed topology networks.
\newblock \emph{Advances in Neural Information Processing Systems}, 30, 2017.

\bibitem[Jin et~al.(2016)Jin, Yuan, Iandola, and Keutzer]{jin2016scale}
Peter~H Jin, Qiaochu Yuan, Forrest Iandola, and Kurt Keutzer.
\newblock How to scale distributed deep learning?
\newblock \emph{arXiv preprint arXiv:1611.04581}, 2016.

\bibitem[Karimireddy et~al.(2019)Karimireddy, Rebjock, Stich, and
  Jaggi]{karimireddy2019error}
Sai~Praneeth Karimireddy, Quentin Rebjock, Sebastian Stich, and Martin Jaggi.
\newblock Error feedback fixes signsgd and other gradient compression schemes.
\newblock In \emph{International Conference on Machine Learning}, pages
  3252--3261. PMLR, 2019.

\bibitem[Karimireddy et~al.(2020)Karimireddy, Kale, Mohri, Reddi, Stich, and
  Suresh]{karimireddy2020scaffold}
Sai~Praneeth Karimireddy, Satyen Kale, Mehryar Mohri, Sashank Reddi, Sebastian
  Stich, and Ananda~Theertha Suresh.
\newblock Scaffold: Stochastic controlled averaging for federated learning.
\newblock In \emph{International conference on machine learning}, pages
  5132--5143. PMLR, 2020.

\bibitem[Khaled et~al.(2020)Khaled, Mishchenko, and
  Richt{\'a}rik]{khaled2020tighter}
Ahmed Khaled, Konstantin Mishchenko, and Peter Richt{\'a}rik.
\newblock Tighter theory for local sgd on identical and heterogeneous data.
\newblock In \emph{International Conference on Artificial Intelligence and
  Statistics}, pages 4519--4529. PMLR, 2020.

\bibitem[Kingma(2014)]{kingma2014adam}
Diederik~P Kingma.
\newblock Adam: A method for stochastic optimization.
\newblock \emph{arXiv preprint arXiv:1412.6980}, 2014.

\bibitem[Kunstner et~al.(2023)Kunstner, Chen, Lavington, and
  Schmidt]{kunstner2023noise}
Frederik Kunstner, Jacques Chen, Jonathan~Wilder Lavington, and Mark Schmidt.
\newblock Noise is not the main factor behind the gap between sgd and adam on
  transformers, but sign descent might be.
\newblock \emph{arXiv preprint arXiv:2304.13960}, 2023.

\bibitem[Lian et~al.(2017)Lian, Zhang, Zhang, Hsieh, Zhang, and
  Liu]{lian2017can}
Xiangru Lian, Ce~Zhang, Huan Zhang, Cho-Jui Hsieh, Wei Zhang, and Ji~Liu.
\newblock Can decentralized algorithms outperform centralized algorithms? a
  case study for decentralized parallel stochastic gradient descent.
\newblock \emph{Advances in neural information processing systems}, 30, 2017.

\bibitem[Lin et~al.(2020)Lin, Stich, Patel, and Jaggi]{lindon}
Tao Lin, Sebastian~U Stich, Kumar~Kshitij Patel, and Martin Jaggi.
\newblock Don't use large mini-batches, use local sgd.
\newblock In \emph{International Conference on Learning Representations}, 2020.

\bibitem[Liu et~al.(2024{\natexlab{a}})Liu, Wu, Chen, Liang, Zhu, Liang,
  Krishnamoorthi, and Liu]{liu2024communication}
Bo~Liu, Lemeng Wu, Lizhang Chen, Kaizhao Liang, Jiaxu Zhu, Chen Liang,
  Raghuraman Krishnamoorthi, and Qiang Liu.
\newblock Communication efficient distributed training with distributed lion.
\newblock \emph{arXiv preprint arXiv:2404.00438}, 2024{\natexlab{a}}.

\bibitem[Liu et~al.(2024{\natexlab{b}})Liu, Li, Hall, Liang, and Ma]{liusophia}
Hong Liu, Zhiyuan Li, David Leo~Wright Hall, Percy Liang, and Tengyu Ma.
\newblock Sophia: A scalable stochastic second-order optimizer for language
  model pre-training.
\newblock In \emph{The Twelfth International Conference on Learning
  Representations}, 2024{\natexlab{b}}.

\bibitem[Loshchilov(2017)]{loshchilov2017decoupled}
I~Loshchilov.
\newblock Decoupled weight decay regularization.
\newblock \emph{arXiv preprint arXiv:1711.05101}, 2017.

\bibitem[Loshchilov and Hutter(2016)]{loshchilov2016sgdr}
Ilya Loshchilov and Frank Hutter.
\newblock Sgdr: Stochastic gradient descent with warm restarts.
\newblock \emph{arXiv preprint arXiv:1608.03983}, 2016.

\bibitem[Lu et~al.(2023)Lu, Li, Zhang, De~Sa, and He]{lumaximizing}
Yucheng Lu, Conglong Li, Minjia Zhang, Christopher De~Sa, and Yuxiong He.
\newblock Maximizing communication efficiency for large-scale training via 0/1
  adam.
\newblock In \emph{The Eleventh International Conference on Learning
  Representations}, 2023.

\bibitem[McMahan et~al.(2017)McMahan, Moore, Ramage, Hampson, and
  y~Arcas]{mcmahan2017communication}
Brendan McMahan, Eider Moore, Daniel Ramage, Seth Hampson, and Blaise~Aguera
  y~Arcas.
\newblock Communication-efficient learning of deep networks from decentralized
  data.
\newblock In \emph{Artificial intelligence and statistics}, pages 1273--1282.
  PMLR, 2017.

\bibitem[Nesterov(1983)]{nesterov1983method}
Yurii Nesterov.
\newblock A method for solving a convex programming problem with convergence
  rate o($1/k^2$.
\newblock In \emph{Soviet Mathematics. Doklady}, volume~27, pages 367--372,
  1983.

\bibitem[Polyak(1964)]{polyak1964some}
Boris~T Polyak.
\newblock Some methods of speeding up the convergence of iteration methods.
\newblock \emph{Ussr computational mathematics and mathematical physics},
  4\penalty0 (5):\penalty0 1--17, 1964.

\bibitem[Radford et~al.(2019)Radford, Wu, Child, Luan, Amodei, Sutskever,
  et~al.]{radford2019language}
Alec Radford, Jeffrey Wu, Rewon Child, David Luan, Dario Amodei, Ilya
  Sutskever, et~al.
\newblock Language models are unsupervised multitask learners.
\newblock \emph{OpenAI blog}, 1\penalty0 (8):\penalty0 9, 2019.

\bibitem[Rajbhandari et~al.(2020)Rajbhandari, Rasley, Ruwase, and
  He]{rajbhandari2020zero}
Samyam Rajbhandari, Jeff Rasley, Olatunji Ruwase, and Yuxiong He.
\newblock Zero: Memory optimizations toward training trillion parameter models.
\newblock In \emph{SC20: International Conference for High Performance
  Computing, Networking, Storage and Analysis}, pages 1--16. IEEE, 2020.

\bibitem[Robbins and Monro(1951)]{robbins1951stochastic}
Herbert Robbins and Sutton Monro.
\newblock A stochastic approximation method.
\newblock \emph{The annals of mathematical statistics}, pages 400--407, 1951.

\bibitem[Safaryan and Richt{\'a}rik(2021)]{safaryan2021stochastic}
Mher Safaryan and Peter Richt{\'a}rik.
\newblock Stochastic sign descent methods: New algorithms and better theory.
\newblock In \emph{International Conference on Machine Learning}, pages
  9224--9234. PMLR, 2021.

\bibitem[Seide et~al.(2014)Seide, Fu, Droppo, Li, and Yu]{seide20141}
Frank Seide, Hao Fu, Jasha Droppo, Gang Li, and Dong Yu.
\newblock 1-bit stochastic gradient descent and its application to
  data-parallel distributed training of speech dnns.
\newblock In \emph{Interspeech}, volume 2014, pages 1058--1062. Singapore,
  2014.

\bibitem[Shi et~al.(2022)Shi, Du, Jordan, and Su]{shi2022understanding}
Bin Shi, Simon~S Du, Michael~I Jordan, and Weijie~J Su.
\newblock Understanding the acceleration phenomenon via high-resolution
  differential equations.
\newblock \emph{Mathematical Programming}, pages 1--70, 2022.

\bibitem[Sun et~al.(2023)Sun, Wang, Li, and Wang]{sun2023momentum}
Tao Sun, Qingsong Wang, Dongsheng Li, and Bao Wang.
\newblock Momentum ensures convergence of signsgd under weaker assumptions.
\newblock In \emph{International Conference on Machine Learning}, pages
  33077--33099. PMLR, 2023.

\bibitem[Sun et~al.(2024)Sun, Qin, Sun, Li, Li, Shen, Qiao, and
  Zhong]{sun2024co2}
Weigao Sun, Zhen Qin, Weixuan Sun, Shidi Li, Dong Li, Xuyang Shen, Yu~Qiao, and
  Yiran Zhong.
\newblock Co2: Efficient distributed training with full
  communication-computation overlap.
\newblock \emph{arXiv preprint arXiv:2401.16265}, 2024.

\bibitem[Sutskever et~al.(2013)Sutskever, Martens, Dahl, and
  Hinton]{sutskever2013importance}
Ilya Sutskever, James Martens, George Dahl, and Geoffrey Hinton.
\newblock On the importance of initialization and momentum in deep learning.
\newblock In \emph{International conference on machine learning}, pages
  1139--1147. PMLR, 2013.

\bibitem[Tang and Chang(2024)]{tang2024fedlion}
Zhiwei Tang and Tsung-Hui Chang.
\newblock Fedlion: Faster adaptive federated optimization with fewer
  communication.
\newblock In \emph{ICASSP 2024-2024 IEEE International Conference on Acoustics,
  Speech and Signal Processing (ICASSP)}, pages 13316--13320. IEEE, 2024.

\bibitem[Vaswani(2017)]{vaswani2017attention}
A~Vaswani.
\newblock Attention is all you need.
\newblock \emph{Advances in Neural Information Processing Systems}, 2017.

\bibitem[Wang et~al.(2019)Wang, Tantia, Ballas, and Rabbat]{wang2019slowmo}
Jianyu Wang, Vinayak Tantia, Nicolas Ballas, and Michael Rabbat.
\newblock Slowmo: Improving communication-efficient distributed sgd with slow
  momentum.
\newblock \emph{arXiv preprint arXiv:1910.00643}, 2019.

\bibitem[Wang et~al.(2020)Wang, Liu, Liang, Joshi, and Poor]{wang2020tackling}
Jianyu Wang, Qinghua Liu, Hao Liang, Gauri Joshi, and H~Vincent Poor.
\newblock Tackling the objective inconsistency problem in heterogeneous
  federated optimization.
\newblock \emph{Advances in neural information processing systems},
  33:\penalty0 7611--7623, 2020.

\bibitem[Xie et~al.(2020)Xie, Zheng, Koyejo, Gupta, Li, and Lin]{xie2020cser}
Cong Xie, Shuai Zheng, Sanmi Koyejo, Indranil Gupta, Mu~Li, and Haibin Lin.
\newblock Cser: Communication-efficient sgd with error reset.
\newblock \emph{Advances in Neural Information Processing Systems},
  33:\penalty0 12593--12603, 2020.

\bibitem[Yu et~al.(2019)Yu, Jin, and Yang]{yu2019linear}
Hao Yu, Rong Jin, and Sen Yang.
\newblock On the linear speedup analysis of communication efficient momentum
  sgd for distributed non-convex optimization.
\newblock In \emph{International Conference on Machine Learning}, pages
  7184--7193. PMLR, 2019.

\bibitem[Zhang et~al.(2019)Zhang, Lucas, Ba, and Hinton]{zhang2019Lookahead}
Michael Zhang, James Lucas, Jimmy Ba, and Geoffrey~E Hinton.
\newblock Lookahead optimizer: k steps forward, 1 step back.
\newblock \emph{Advances in neural information processing systems}, 32, 2019.

\end{thebibliography}
